\documentclass{article}

\usepackage{microtype}
\usepackage{graphicx}
\usepackage{subcaption}
\usepackage{booktabs} 
\usepackage{amsmath}
\usepackage{multirow,multicol}

\usepackage{pgfplots}
\usepackage{pgfplotstable}
\usepgfplotslibrary{groupplots}

\usepackage{hyperref}

\usepackage[accepted]{icml2026}

\usepackage{xspace}

\usepackage[most]{tcolorbox}
\usepackage{mathtools}

\newcommand{\hn}[1]{{{\color{purple}{#1}}}}

\newcommand{\Dcal}{\mathcal{D}}

\newcommand{\Lcal}{\mathcal{L}}
\newcommand{\Bcal}{\mathcal{B}}
\newcommand{\Scal}{\mathcal{S}}
\newcommand{\Ecal}{\mathcal{E}}
\newcommand{\ecal}{\varepsilon}
\newcommand{\lcal}{\ell}

\newcommand{\Embb}{\mathbb{E}}
\newcommand{\Rmbb}{\mathbb{R}}

\newcommand{\Smbb}{\mathbb{S}}

\newcommand{\CBS}{\mathrm{CBS}}

\DeclareMathOperator{\sign}{sign}
\DeclareMathOperator{\matsign}{matsign}
\DeclareMathOperator{\tr}{tr}

\DeclareMathOperator{\rank}{rank}

\pgfplotscreateplotcyclelist{myColors}{
    {orange, very thick},
    {olive, very thick},
    {blue, very thick},
}
\pgfplotsset{
    Llama3_perplexity_plots/.style={
        height=4cm, width=4.3cm, 
        grid=major,
        scaled x ticks=false, 
        label style={font=\scriptsize},      
        tick label style={font=\scriptsize}, 
        tick style={draw=none},
        ylabel style={yshift=-1.5em},
        xlabel style={yshift=0.5em},
        ymode=log,
        ymin=20, ymax = 1e4,
        cycle list name=myColors, 
        legend columns=3,             
        legend style={font=\scriptsize, column sep=10pt} 
    }
}

\pgfplotsset{
    Llama3_dissect_plots/.style={
        height=4.5cm, width=4.5cm, 
        grid=major,
        scaled x ticks=false, 
        label style={font=\scriptsize},      
        tick label style={font=\scriptsize}, 
        tick style={draw=none},
        ylabel style={yshift=-1.5em},
        xlabel style={yshift=0.5em},
        cycle list name=color list, 
    }
}

\pgfplotsset{
  Imagewoof_val_loss_plots/.style={
    height=4.5cm,width=4.5cm,
    grid=major,
    scaled x ticks=false,
    label style={font=\scriptsize},
    tick label style={font=\scriptsize},
    tick style={draw=none},
    ylabel style={yshift=-1.5em},
    xlabel style={yshift=0.5em},
    cycle list name=color list,
    every axis plot/.append style={line width=1.1pt},
    every axis plot post/.append style={mark=none},
    legend image post style={mark=none},
    legend style={at={(0.65,0.98)},anchor=north,draw=black,fill=white,fill opacity=0.9,text opacity=1,font=\scriptsize,inner sep=2pt,row sep=-2pt,cells={anchor=west}},
    xmin=0,xmax=15000,
    xtick={0,5000,10000,15000},
    xticklabels={0,5k,10k,15k},
    ymode=normal,
    ymin=0.5,ymax=2.5,
    ytick={1,1.5,2,2.5,3},
    yticklabel style={/pgf/number format/fixed,/pgf/number format/precision=1}
  }
}

\pgfplotsset{
  Imagewoof_val_loss_plots_gns_type/.style={
    height=4.5cm,width=4.5cm,
    grid=major,
    scaled x ticks=false,
    label style={font=\scriptsize},
    tick label style={font=\scriptsize},
    tick style={draw=none},
    ylabel style={yshift=-1.5em},
    xlabel style={yshift=0.5em},
    cycle list name=color list,
    every axis plot/.append style={line width=1.1pt},
    every axis plot post/.append style={mark=none},
    legend image post style={mark=none},
    legend style={at={(0.7,0.98)},anchor=north,draw=black,fill=white,fill opacity=0.9,text opacity=1,font=\scriptsize,inner sep=2pt,row sep=-2pt,cells={anchor=west}},
    xmin=0,xmax=3000,
    xtick={0,1000,2000,3000},
    xticklabels={0,1k,2k,3k},
    ymode=normal,
    ymin=0.5,ymax=3,
    yticklabel style={/pgf/number format/fixed,/pgf/number format/precision=1}
  }
}

\definecolor{colorA}{RGB}{68, 1, 84}    
\definecolor{colorB}{RGB}{59, 82, 139}  
\definecolor{colorC}{RGB}{33, 145, 140} 
\definecolor{colorD}{RGB}{94, 201, 98}  
\definecolor{colorE}{RGB}{253, 231, 37} 

\newcommand{\sgd}{SGD\xspace}
\newcommand{\msgd}{MSGD\xspace}
\newcommand{\signsgd}{signSGD\xspace}
\newcommand{\signum}{Signum\xspace}
\newcommand{\adam}{Adam\xspace}
\newcommand{\adamw}{AdamW\xspace}
\newcommand{\muon}{Muon\xspace}
\newcommand{\specgd}{specSGD\xspace}

\newcommand{\nsgd}{NormalizedSGD\xspace}

\usepackage{amsmath}
\usepackage{amssymb}
\usepackage{mathtools}
\usepackage{amsthm}
\usepackage[dvipsnames,table]{xcolor}
\colorlet{SGDColor}{Gray!10}
\colorlet{SignSGDColor}{Maroon!10}
\colorlet{SpecSGDColor}{MidnightBlue!10}

\usepackage[capitalize,noabbrev]{cleveref}
\usetikzlibrary{arrows.meta, backgrounds, shadows.blur}

\theoremstyle{plain}
\newtheorem{theorem}{Theorem}[section]

\newtheorem{lemma}[theorem]{Lemma}

\theoremstyle{definition}
\newtheorem{definition}[theorem]{Definition}

\theoremstyle{remark}

\usepackage[textsize=tiny]{todonotes}

\icmltitlerunning{Adaptive Batch Sizes Using Non-Euclidean Gradient Noise Scales for Stochastic Sign and Spectral Descent}

\begin{document}

\twocolumn[
  \icmltitle{Adaptive Batch Sizes Using Non-Euclidean Gradient Noise Scales for Stochastic Sign and Spectral Descent}



  \icmlsetsymbol{equal}{*}

  \begin{icmlauthorlist}
    \icmlauthor{Hiroki Naganuma}{mila,udem}
    \icmlauthor{Shagun Gupta}{meta}
    \icmlauthor{Youssef Briki}{udem}
    \icmlauthor{Ioannis Mitliagkas}{mila,udem}
    \icmlauthor{Irina Rish}{mila,udem}
    \icmlauthor{Parameswaran Raman}{meta}
    \icmlauthor{Hao-Jun Michael Shi}{meta}
  \end{icmlauthorlist}

  \icmlaffiliation{mila}{Mila, Montreal, Canada}
  \icmlaffiliation{meta}{Meta Platforms, Menlo Park, California, USA}
  \icmlaffiliation{udem}{Université de Montréal, Montreal, Canada}

  \icmlcorrespondingauthor{Hiroki Naganuma}{naganuma.hiroki@mila.quebec}
  \icmlcorrespondingauthor{Shagun Gupta}{shagun@meta.com}

  \icmlkeywords{Machine Learning, ICML}

  \vskip 0.3in
]



\printAffiliationsAndNotice{}  

\begin{abstract}
To maximize hardware utilization, modern machine learning systems typically employ large constant or manually tuned batch size schedules, relying on heuristics that are brittle and costly to tune. Existing adaptive strategies based on gradient noise scale (GNS) offer a principled alternative. However, their assumption of SGD's Euclidean geometry creates a fundamental mismatch with popular optimizers based on generalized norms, such as \signsgd{} / \signum ($\ell_\infty$) and stochastic spectral descent (\specgd) / \muon ($\mathcal{S}_\infty$). In this work, we derive gradient noise scales for \signsgd and \specgd that naturally emerge from the geometry of their respective dual norms. To practically estimate these non-Euclidean metrics, we propose an efficient variance estimation procedure that leverages the local mini-batch gradients on different ranks in distributed data-parallel systems. Our experiments demonstrate that adaptive batch size strategies using non-Euclidean GNS enable us to match the validation loss of constant-batch baselines while reducing training steps by up to 66\% for \signum and \muon on a 160 million parameter Llama model.
\end{abstract}


\section{Introduction} 
\label{sec:intro}

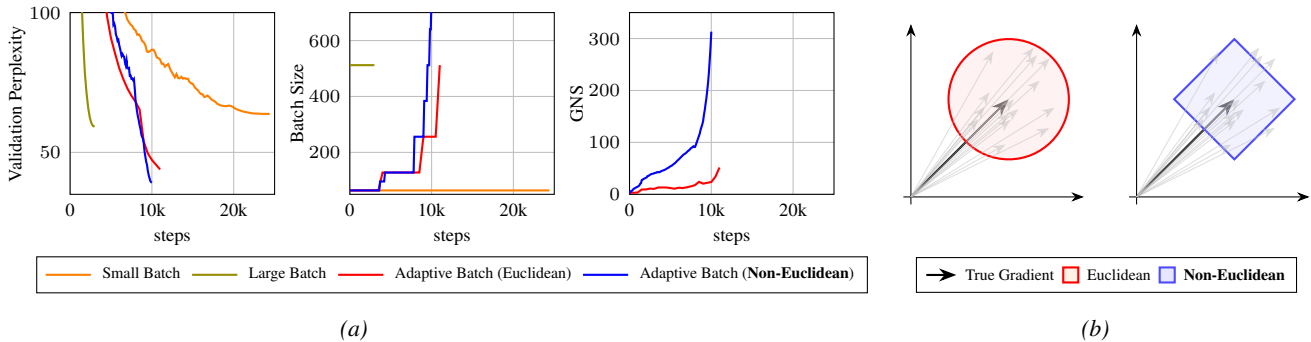
\begin{figure*}[ht]
    \centering
    \begin{subfigure}{0.55\textwidth}
        \centering
        \begin{tikzpicture} 
            \begin{groupplot}[
                Llama3_dissect_plots,     
                height=4cm, width=4.3cm, 
                legend columns=-1,          
                legend style={
                    font=\tiny,
                    fill opacity=1,
                    column sep=2pt,         
                },
                group style={
                    group size=3 by 1,       
                    horizontal sep=1cm,    
                }
            ]
                \nextgroupplot[
                    restrict y to domain=0:200,
                    xlabel={steps},
                    ylabel={Validation Perplexity},
                    ymin=35, ymax=100,
                    xmin=0, xmax=12000,
                    ytick={50, 100},
                    xtick={0, 5000, 10000},
                    xticklabels={0, 5k, 10k}, 
                    legend to name=CommonLegend, 
                ]
                \addplot[orange, thick] table [x=_step, y=perplexity, col sep=comma] {Llama_results/introduction/baseline_64.csv};
                \addlegendentry{Small Batch}
                \addplot[olive, thick] table [x=_step, y=perplexity, col sep=comma] {Llama_results/introduction/baseline_512.csv};
                \addlegendentry{Large Batch}
                \addplot[red, thick] table [x=_step, y=perplexity, col sep=comma] {Llama_results/introduction/gns_l2.csv};
                \addlegendentry{Adaptive Batch (Euclidean)}
                \addplot[blue, thick] table [x=_step, y=perplexity, col sep=comma] {Llama_results/introduction/gns_l1.csv};
                \addlegendentry{Adaptive Batch (\textbf{Non-Euclidean})}
                
                \nextgroupplot[
                    xlabel={steps},
                    ylabel={Batch Size},
                    ymin=50, ymax=800,
                    xmin=0, xmax=12000,
                    xtick={0, 5000, 10000},
                    xticklabels={0, 5k, 10k}, 
                ]
                \addplot[orange, thick] table [x=_step, y=batch_size/global_batch_size, col sep=comma] {Llama_results/introduction/baseline_64.csv};
                \addplot[olive, thick] table [x=_step, y=batch_size/global_batch_size, col sep=comma] {Llama_results/introduction/baseline_512.csv};
                \addplot[red, thick] table [x=_step, y=batch_size/global_batch_size, col sep=comma] {Llama_results/introduction/gns_l2.csv};
                \addplot[blue, thick] table [x=_step, y=batch_size/global_batch_size, col sep=comma] {Llama_results/introduction/gns_l1.csv};
                
                \nextgroupplot[
                    xlabel={steps},
                    ylabel={GNS},
                    ymin=0, ymax=350,
                    xmin=0, xmax=12000,
                    xtick={0, 5000, 10000},
                    xticklabels={0, 5k, 10k}, 
                ]
                \addplot[red, thick] table [x=_step, y=adaptive_batch_metrics/gns_ema, col sep=comma] {Llama_results/introduction/gns_l2.csv};
                \addplot[blue, thick] table [x=_step, y=adaptive_batch_metrics/gns_ema, col sep=comma] {Llama_results/introduction/gns_l1.csv};
            \end{groupplot}
            \node[below=0.7cm] at (group c2r1.south) {\pgfplotslegendfromname{CommonLegend}};
        \end{tikzpicture}
        \caption{}
        \label{fig:GNS_intro_plots}
    \end{subfigure}
    \hfill
    \begin{subfigure}{0.3\textwidth}
        \begin{tikzpicture}[>=Stealth]
            \path (0,4) -- (0,0);
            \def\yc{1.4}
            \def\tgrad{1.3}
            \def\rad{0.8}
            \def\axis{2.3}
            \def\xc{0}
            \def\xcc{3}

            \draw[->] (\xc-0.1,\yc) -- (\xc+\axis,\yc);
            \draw[->] (\xc,\yc-0.1) -- (\xc,\yc+\axis);
            \draw[->] (\xcc-0.1,\yc) -- (\xcc+\axis,\yc);
            \draw[->] (\xcc,\yc-0.1) -- (\xcc,\yc+\axis);

            \draw[thick, black, ->] (\xc, \yc) -- (\xc + \tgrad, \yc + \tgrad);
            \draw[thick, black, ->] (\xcc, \yc) -- (\xcc + \tgrad, \yc + \tgrad);
    
            \draw[thick, red, fill=red!10, fill opacity=0.5] (\xc+\tgrad,\yc + \tgrad) circle (\rad);
    
            \draw[thick, blue!70, fill=blue!10, fill opacity=0.5] 
                    (\xcc+\tgrad+\rad, \yc+\tgrad) -- 
                    (\xcc+\tgrad, \yc+\tgrad+\rad) -- 
                    (\xcc+\tgrad-\rad, \yc+\tgrad) -- 
                    (\xcc+\tgrad, \yc+\tgrad-\rad) -- cycle;

            \def\numsamples{15}
            \def\anglespread{25} 
            \def\magspread{0.7} 
            
            \foreach \i in {1,...,\numsamples} {
                \pgfmathsetmacro{\rAngle}{45 + \anglespread*rand}
                \pgfmathsetmacro{\rMag}{\tgrad*1.5 + \magspread*rand}
                
                \draw[gray!50, thin, ->, opacity=0.5] (\xc, \yc) -- 
                    ({\xc + \rMag*cos(\rAngle)}, {\yc + \rMag*sin(\rAngle)});
                \draw[gray!50, thin, ->, opacity=0.5] (\xcc, \yc) -- 
                    ({\xcc + \rMag*cos(\rAngle)}, {\yc + \rMag*sin(\rAngle)});
            }
            \matrix [
                draw, 
                below, 
                yshift=0.6cm, 
                column sep=2pt, 
                nodes={
                    anchor=center, 
                    font=\tiny,
                    inner sep=1pt,
                }
            ] at (current bounding box.south) {
                \draw[black, thick, ->] (0,0) -- (0.4,0); & \node{True Gradient}; & 
                \node[draw=red, thick, fill=orange!10, minimum size=2mm] {}; & \node{Euclidean}; &
                \node[draw=blue!70, thick, fill=blue!10, minimum size=2mm] {}; & \node{\textbf{Non-Euclidean}}; \\
            };
        \end{tikzpicture}
        \caption{}
        \label{fig:GNS_geometry}
    \end{subfigure}
    \caption{
    In \textbf{\ref{fig:GNS_intro_plots} (left)}, we plot the validation perplexity, batch size, and the exponential moving average of the GNS over steps when training the 160M parameter Llama 3 model for 3.2B tokens on C4 data using \signsgd. 
    We compare small-batch ($B = 64$), large-batch ($B = 512$), and adaptive batch size strategies based on Euclidean and non-Euclidean GNS. 
    The plots highlight the improved efficiency of using non-Euclidean GNS measurements for \signsgd. 
    In \textbf{\ref{fig:GNS_geometry} (right)}, we depict how enforcing different GNS ($\ell_2$ versus $\ell_1$) changes the geometry of the allowable region for stochastic gradients (with high probability).
    }
    \label{fig:introduction_figure}
\end{figure*}

As model architectures and datasets scale exponentially, maximizing hardware utilization has become a critical objective in training modern machine learning models. This is typically achieved by scaling the batch size, which improves GPU throughput by processing a greater number of samples in parallel and reduces the total number of optimizer steps, ultimately shortening the end-to-end wall-clock training time \citep{smith2017don,mccandlish2018empirical}.
Despite these benefits, constant large batch sizes suffer from diminishing returns in per-sample improvement. This reduction in sample efficiency compromises model quality under fixed training budgets \citep{keskar2016large,smith2017don,shallue2019measuring,nado2021large,marek2025small}.

To manage this trade-off, practitioners employ heuristic schedules that increase the batch size during training, such as through \emph{stage-wise} \citep{rae2021scaling,hoffmann2022training,chowdhery2023palm,grattafiori2024llama,groeneveld2024olmo} or \emph{linear ramp-ups} \citep{mann2020language,smith2022using,liu2024deepseek}.
These methods attempt to balance small-batch efficiency with large-batch throughput; yet, they are brittle and demand significant tuning.
Without a principled, computable metric to guide the schedule, practitioners lack a reliable way of predicting when increasing the batch size will compromise sample efficiency.

A more principled alternative involves dynamically choosing the batch size by monitoring the \emph{gradient noise scale} (GNS), a metric that quantifies the relative error in the gradient estimate \citep{byrd2012sample, mccandlish2018empirical, chen2025divebatch}. 
By adjusting the batch size to track this signal, one can avoid the diminishing return associated with exceeding the \emph{critical batch size} (CBS).
Although various formalizations of the CBS exist, they share the objective of identifying the threshold above which further increases in batch size yield negligible gains in per-sample improvement. 
Theoretically, this threshold is defined as a constant multiple of the GNS (see \Cref{sec:alg}). 
In this work, we focus on designing GNS metrics to characterize and implement adaptive batch size strategies.

While previous research derived the CBS and GNS based on the descent lemma for the stochastic gradient method (\sgd) \citep{mccandlish2018empirical} and its underlying Euclidean geometry, modern vision and language workloads favor optimizers that operate in \emph{non-Euclidean geometries}, such as sign-based methods \citep{bernstein2018signsgd} and spectral methods \cite{carlson2015stochastic,carlson2015preconditioned,bernstein2024old,jordan2024muon}.
Furthermore, popular optimizers like Adam \citep{kingma2014adam}, Shampoo \citep{gupta2018shampoo}, and SOAP \citep{vyas2024soap,eschenhagen2025purifying} have been increasingly interpreted from this non-Euclidean lens, sharing fundamental geometric properties with sign and spectral updates \citep{balles2018dissecting,orvieto2025search}.
Consequently, the current literature suffers from a fundamental misalignment: using Euclidean-based GNS metrics to determine batch sizes dynamically for non-Euclidean optimizers.

To bridge this gap, we derive metrics for stochastic sign descent (\signsgd) and stochastic spectral descent (\specgd) that measure the gradient noise scale in the dual norm of the optimizer's geometry (i.e., $\ell_1$ norm for sign-based optimizers and nuclear norm for spectral optimizers).
When analyzed with respect to their corresponding geometry, the primary source of error in \signsgd and \specgd is induced by the \emph{bias} in the search direction, which can be bounded by the expected dual norm of the gradient error.
Leveraging these refined GNS statistics yields substantial empirical benefits in practice, as illustrated in \Cref{fig:introduction_figure}.

Efficiently estimating GNS metrics for adaptive batch sizes remains a challenge due to high memory and computational overhead.
Prior approaches that calculate variance using individual samples are memory-intensive and untenable for large-scale training \citep{bollapragada2018adaptive, chen2025divebatch}.
Alternative variance estimators that compare local and global batches are practical but high variance since they rely on a single sample \citep{mccandlish2018empirical, merrill2025critical}.
Instead, we propose to use the local mini-batch gradients across data-parallel ranks as multiple independent samples, appending partial statistics to the gradient \texttt{AllReduce} operation required during distributed training.
This approach leverages the distributed frameworks (e.g., DDP \cite{li2020pytorch} and FSDP \citep{zhao2023pytorch}) already intrinsic to most large machine learning systems to provide a reliable real-time GNS signal.

\paragraph{Contributions.} In this work, we derive \emph{non-Euclidean GNS} metrics for sign-based (\signsgd, \signum) and spectral optimizers (\specgd, \muon). 
We propose an adaptive batch size strategy and a scalable variance estimation procedure tailored for distributed training setups. 
We demonstrate that adaptive batch size strategies guided by these non-Euclidean metrics yield up to a 66\% \emph{reduction in optimizer steps} for language and vision tasks. 
Crucially, our schedules match the validation loss of constant-batch baselines while significantly improving training efficiency.

\paragraph{Conflict of Interest Disclosure.} The authors declare that they have no relevant financial or substantive conflicts of interest to disclose.

\section{Background}
\label{sec:background}

We formulate the training objective as the expected risk minimization problem:
\begin{equation}
    \label{eq:problem}
    \min_{x \in \Rmbb^d}\Lcal(x) = \Embb_{\xi \sim \Dcal} \left[\lcal(x; \xi)\right]
\end{equation}
where $\xi$ is drawn from an underlying stationary distribution $\Dcal$ with support $\Xi$, $\lcal: \Rmbb^d \times \Xi \rightarrow \Rmbb$ is the loss function evaluated on sample $\xi$, and $\Embb$ denotes the expectation with respect to $\Dcal$. 

\subsection{Stochastic Generalized Steepest Descent}
\label{sec:steepest_descent}

We consider stochastic mini-batch variants of the \emph{generalized steepest descent method}. 
At each iteration $k$, the method samples a mini-batch $\Scal_k$ of $B_k$ i.i.d. samples from $\Dcal$, then computes the mini-batch gradient $g_k = \tfrac{1}{B_k}\sum_{\xi \in \Scal_k} \nabla \lcal(x_k; \xi) \in \Rmbb^d$. 
The current iterate $x_k$ is updated via:
\begin{equation} 
    \label{eq:general_update_rule}
    x_{k+1} = x_k - \eta_k p_k, ~~ \text{where} ~~ p_k = \arg \max_{\|p\| \leq 1} \langle g_k, p \rangle.
\end{equation}
The search direction $p_k$ is the steepest descent direction for the stochastic gradient estimate over the chosen norm $\|\cdot\|$ (with corresponding dual norm $\|\cdot\|_{*}$). 
As noted in \citet{wright1999numerical,boyd2004convex} and \citet{beck2017first},
different choices for primal and dual norms ($\|\cdot\|$  / $\|\cdot\|_*$) recover standard optimization algorithms: 
\begin{itemize}
    \item $\ell_2$ norm ($\|\cdot\|_2$ / $\|\cdot \|_2$): Yields normalized SGD, where $p_k = g_k / \|g_k\|_2$.
    \item $\ell_\infty$ norm ($\|\cdot\|_\infty$ / $\|\cdot\|_1$): Yields stochastic sign descent (\signsgd), where $p_k = \sign(g_k)$.
    \item Spectral or Schatten-$\infty$ norm ($\|\cdot \|_{\Scal_\infty}$ / $\|\cdot\|_{\Scal_1}$): For matrix variables $X \in \Rmbb^{m \times n}$, yields stochastic spectral descent (\specgd) with $P_k = \matsign(G_k) = U_k V_k^\top$.\footnote{The Schatten-$1$ norm is the nuclear norm $\|G_k\|_{\Scal_1} = \tr(C_k)$.}
\end{itemize}
Here, the matrix sign function is defined via the reduced SVD $G_k = U_k \Sigma_k V_k^\top$ and taking the sign function of the singular values, i.e., $P_k = U_k V_k^\top$. 
Geometrically, $P_k$ corresponds to the unitary factor of the polar decomposition of $G_k$, which is the closest semi-orthogonal matrix to $G_k$ in Frobenius norm.\footnote{One may also consider un-normalized variants of these methods, which additionally multiply $p_k$ by the dual norm of the gradient $\|g_k\|_*$. In the deterministic setting, these methods share a unified analysis based on smoothness with respect to arbitrary norms $\|\cdot \|$ \citep{balles2020geometry}. However, besides SGD, these variants are not commonly used in practice.}

Other popular optimizers such as \adam \citep{kingma2014adam} and \signsgd with momentum (also known as \signum) \citep{bernstein2018signsgd} as well as Shampoo \citep{gupta2018shampoo} and \muon \citep{jordan2024muon} are known to reduce to \signsgd and \specgd as special cases, respectively \citep{balles2018dissecting, bernstein2024old, orvieto2025search}. 
While \specgd was originally suggested for training neural networks by \citet{carlson2015stochastic,carlson2015preconditioned}, it has recently resurged in popularity due to its connection to the theory of modular duality \citep{bernstein2024modular} and the \muon optimizer \citep{jordan2024muon,liu2025muon}, which utilizes momentum and the Newton-Schulz iteration to efficiently approximate the semi-orthogonalization of the gradient.
These interpretations motivate our study of adaptive batch sizes for these fundamental methods.

\subsection{Critical Batch Size and Gradient Noise Scale}
\label{sec:background_gns}

Training with large batch sizes is known to degrade model quality and sample efficiency after exceeding a certain point, often referred to as the \emph{critical batch size} (CBS) \citep{keskar2016large, shallue2019measuring, mccandlish2018empirical}. 
Theoretically, the CBS is defined as the batch size that minimizes the total computational complexity (e.g., stochastic first-order oracle complexity). 
Recent work has applied this framework to derive closed-form CBS expressions for specific optimizers like \muon \citep{sato2025analysis}. 

To mitigate this degradation in practice, various heuristics have been proposed. 
Most notable among these are learning rate warmup and scaling rules: \emph{linear scaling} \citep{goyal2017accurate, smith2017don} and \emph{square root scaling} \citep{krizhevsky2014one, hoffer2017train}. 
Additionally, layerwise adaptive methods like LARS \citep{you2017large} and LAMB \citep{you2019large} were developed to stabilize training by adjusting per-layer learning rates based on the ratio of weight to gradient norms.


Despite the success of these techniques, training with a fixed batch size inevitably encounters a regime of diminishing returns. 
Since the number of samples consumed by each optimizer step scales with the batch size, the optimizer must make more progress per-step to maintain sample efficiency. 
While this is typically achieved by scaling the learning rate to leverage the gradient's reduced variance, this strategy is limited by the deterministic properties of the loss landscape, i.e., the maximum curvature \citep{wright1999numerical, mandt2017stochastic,gilmer2021loss}. 
Consequently, further reducing the gradient variance by increasing the batch size beyond the CBS provides little additional benefit per step.


\paragraph{Euclidean gradient noise scale.} The seminal work by \citet{mccandlish2018empirical} derives a formal notion of CBS based on the optimal expected single-step improvement in the loss function for (un-normalized) SGD.\footnote{Note that CBS is an overloaded term: In the systems literature, it often refers to the maximum efficient batch size for an entire training run \citep{shallue2019measuring}. Here, we refer to the \emph{instantaneous} CBS at a specific optimization step $k$.}
Let $C_k \in \Smbb^{d}_+$ denote the gradient covariance matrix at step $k$: 
\begin{equation*}
    C_k = \Embb_k[(\nabla \ell(x_k; \xi) - \nabla \Lcal(x_k)) (\nabla \ell(x_k; \xi) - \nabla \Lcal(x_k))^\top]
\end{equation*}
where $\Embb_k[\cdot] = \Embb[ ~\cdot \mid x_k]$ denotes the conditional expectation given $x_k$.
By leveraging a quadratic approximation of the loss function (assuming $\nabla^2 \Lcal(x_k) \approx I$ \cite{mccandlish2018empirical}):
\begin{equation} \label{eq:loss_change_initial_approximation}
    \Lcal(x_{k + 1}) = \Lcal(x_k) - \eta_k \langle \nabla \Lcal(x_k), g_k \rangle + \frac{\eta_k^2}{2} \| g_k \|_2^2,
\end{equation}
one can derive the \emph{expected one-step improvement} for SGD:
\begin{equation}
    \begin{split}
        & \Delta \Lcal(x_k, \eta_k, \| \cdot \|_2, B_k) = \Embb_k [\Lcal(x_k) - \Lcal(x_{k + 1})] \\
        & ~~ = \eta_k \|\nabla \Lcal(x_k)\|_2^2 - \frac{\eta_k^2}{2} \left(\|\nabla \Lcal(x_k)\|_2^2 + \frac{\tr(C_k)}{B_k} \right).
    \end{split}
    \label{eq:SGD_expected_loss_improv}
\end{equation}
Here, the equality follows from the bias-variance decomposition $\Embb_k[\|g_k\|_2^2] = \|\nabla \Lcal(x_k)\|_2^2 + \Embb_k[\|g_k - \nabla \Lcal(x_k)\|_2^2]$.

Maximizing the expected improvement in \Cref{eq:SGD_expected_loss_improv} with respect to $\eta_k$ yields the optimal learning rate:
\begin{equation}
    \label{eq:sgd_optimal_lr}
    \eta_k^*(B_k) = \left(1 + \frac{1}{B_k} \cdot \frac{\tr(C_k)}{\|\nabla \Lcal(x_k)\|_2^2} \right)^{-1}.
\end{equation}
Plugging $\eta_k^*(B_k)$ into \Cref{eq:SGD_expected_loss_improv} yields the optimal expected improvement, which is proportional to $\eta_k^*(B_k)$. 
Intrinsic to both quantities is the \emph{Euclidean GNS}, denoted as $\Bcal_{\lcal_2}(x_k)$, which governs the trade-off between the batch size and the realizable progress:
\begin{tcolorbox}[
    enhanced,
    colback=white,       
    colframe=black,      
    coltitle=black,      
    title=\textbf{Euclidean ($\ell_2$) GNS for \sgd},
    sharp corners,       
    boxrule=0.8pt,       
    bottom=4mm,
    attach boxed title to top left={xshift=0.5cm, yshift*=-0.8\baselineskip},
    boxed title style={
        frame hidden,    
        colback=white,   
        left=1mm, right=1mm 
    }
]
\begin{equation} \label{eq:l2_GNS}
    \Bcal_{\lcal_2}(x_k) = \frac{\tr(C_k)}{\|\nabla \Lcal(x_k)\|_2^2}.
\end{equation}
\end{tcolorbox}

Practitioners utilize this metric ($\Bcal = \Bcal_{\ell_2}$) to select the batch size:
\begin{tcolorbox}[
    enhanced,
    colback=white,       
    colframe=black,      
    coltitle=black,      
    title=\textbf{GNS $\rightarrow$ Batch Size Rule},
    sharp corners,       
    boxrule=0.8pt,       
    bottom=4mm,
    attach boxed title to top left={xshift=0.5cm, yshift*=-0.7\baselineskip},
    boxed title style={
        frame hidden,    
        colback=white,   
        left=1mm, right=1mm 
    }
]
\begin{equation} \label{eq:gns_to_batch_size}
    B_k = \theta^{-2} \Bcal(x_k), ~~~ \theta \in (0, 1).
\end{equation}
\end{tcolorbox}
The tolerance parameter $\theta \in (0, 1)$ enforces a constant noise-to-signal ratio in the gradient estimate.
This rule is equivalent to the \emph{norm test} proposed in the classical optimization literature to ensure linear convergence for strongly convex functions \citep{byrd2012sample}:
\begin{equation*}
    \Embb [\|g_k - \nabla \Lcal(x_k)\|_2^2] \leq \theta^2 \|\nabla \Lcal(x_k)\|_2^2.
\end{equation*}
Alternative adaptive strategies utilizing gradient diversity \citep{chen2025divebatch}, angular distances \citep{bollapragada2018adaptive,bollapragada2018progressive, bahamou2019dynamic, xu2020dynamically}, local branch training \citep{merrill2025critical}, composite metrics \citep{belias2025one}, and the learning rate schedule \citep{devarakonda2017adabatch,smith2017don} have also been proposed.

While this analysis provides a robust framework for SGD relying on the Euclidean structure in \eqref{eq:SGD_expected_loss_improv}, recent attempts to apply it to non-Euclidean methods \citep{gray2024normalization,lau2024adaptive} inherently ignore the optimizer's geometry.
In the following section, we show how to formally extend these ideas to \signsgd and \specgd.



\section{Non-Euclidean GNS} 
\label{sec:alg}

To generalize the theory of CBS and GNS beyond Euclidean settings, we observe that the \emph{dual norm} of the optimizer is the natural metric for quantifying noise in generalized stochastic steepest descent.
This relationship is obscured in the Euclidean case because the $\ell_2$ norm is its own dual.

For \signsgd and \specgd, the stochastic steepest descent direction is \emph{biased}, i.e., $\Embb_k[p_k] \neq \arg \max_{\|p\| \leq 1} \langle \nabla \Lcal(x_k), p\rangle$. 
We control this bias term by controlling the GNS. 
To make this problem tractable, we derive the GNS by maximizing a lower bound on the expected single-step improvement:
\begin{lemma}
    \label[lemma]{lem:general_norm_inner_product}
     Let $\|\cdot\|$ be a general norm and $\|\cdot\|_*$ be its associated dual norm. The stochastic steepest descent direction $p_k = \arg\max_{\|p\| \leq 1} \langle g_k, p\rangle$ satisfies:
    \begin{equation*}
        \Embb_k \left[\langle \nabla \Lcal (x_k), p_k \rangle \right] \geq \left\|\nabla \Lcal(x_k)\right\|_* - \Embb_k \left[\left\|\nabla \Lcal (x_k) - g_k\right\|_* \right].
    \end{equation*}
\end{lemma}
The proof for \Cref{lem:general_norm_inner_product} is provided in Appendix~\ref{sec:theory}.

Again using a quadratic approximation of the loss function, we generalize \Cref{eq:loss_change_initial_approximation} as: 
\begin{equation}
    \label{eq:general_loss_change}
    \Lcal(x_{k + 1}) = \Lcal(x_k) - \eta_k \langle \nabla \Lcal(x_k), p_k \rangle + \frac{\eta_k^2}{2} \langle p_k, p_k \rangle.
\end{equation}

Taking the expectation of \Cref{eq:general_loss_change} given $x_k$ and using \Cref{lem:general_norm_inner_product}, we obtain the key inequality:
\begin{equation}
    \label{eq:non_euclidean_improvement_initial_bound}
    \begin{split}
        &\Delta \Lcal (x_k, \eta_k, \|\cdot\|, B_k)  \\
        & ~~ \geq \eta_k \left\|\nabla \Lcal(x_k)\right\|_* - \eta_k \Embb_k\left[\|\nabla \Lcal (x_k) - g_k\|_*\right] \\
        & \quad \quad - \frac{\eta_k^2}{2} \Embb_k\left[\langle p_k, p_k \rangle \right].
    \end{split}
\end{equation}
Since the update has unit norm, the quadratic term $\Embb_k\left[\langle p_k, p_k \rangle \right]$ is constant and easily evaluated (e.g., $d$ for \signsgd).
Therefore, the primary challenge is in evaluating the expectation of the gradient estimation error measured in the dual norm, $\Embb_k\left[\|\nabla \mathcal{L} (x_k) - g_k\|_* \right]$. 
We characterize this error for \signsgd and \specgd to derive their respective GNS.

\subsection{Stochastic Sign Descent (\signsgd)}

As seen in \cref{sec:steepest_descent}, the search direction $p_k = \sign(g_k)$ corresponds to the steepest descent direction under the $\ell_\infty$ norm. 
Let $\sigma_k \in \Rmbb^d$ be the vector consisting of the component-wise standard deviations of the single-sample gradient $\nabla \ell(x_k; \xi)$. That is, for each component $i = 1, ..., d$:
\begin{equation*}
    [\sigma_k]_i^2 = \Embb \left[([\nabla \ell(x_k; \xi) - \nabla \Lcal(x_k)]_i)^2 \right].
\end{equation*}

To bound the key inequality in \Cref{eq:non_euclidean_improvement_initial_bound}, we utilize the following inequality which bounds the expected $\ell_1$ norm of the gradient error \citep{bernstein2018signsgd}:\footnote{We provide complete statements and proofs in Appendix~\ref{sec:theory}.}
\begin{equation} 
\label{eq:signSGD_error_bound}
    \Embb_k\left[\|\nabla \Lcal(x_k) - g_k\|_1 \right] \leq \frac{\|\sigma_k\|_1}{\sqrt{B_k}}.
\end{equation}
The dependence on $\sqrt{B_k}$ ensures that the error vanishes as the batch size increases. One could obtain a tighter bound on the error by estimating $\Embb_k\left[\|\nabla \Lcal(x_k) - g_k\|_1 \right]$ directly.

Substituting \cref{eq:signSGD_error_bound} and the identity $\langle \sign(g_k), \sign(g_k) \rangle \leq d$ into \Cref{eq:non_euclidean_improvement_initial_bound}, we obtain the specific lower bound for \signsgd:
\begin{equation}
\label{eq:signSGD_expected_loss_improv}
    \begin{split}
        & \Delta \Lcal (x_k, \eta_k, \|\cdot\|_\infty, B_k) \\ 
        & \quad \quad \geq \eta_k \left( \left\|\nabla\Lcal(x_{k})\right\|_1  - \frac{\|\sigma_k\|_1}{\sqrt{B_k}} \right) - \frac{d \eta_k^2}{2}.
    \end{split}
\end{equation}    
The learning rate that maximizes this lower bound is:
\begin{equation}
    \label{eq:sign_optimal_lr}
    \eta_k^*(B_k) = \frac{\left\|\nabla\Lcal(x_{k})\right\|_1}{d} \left(1 - \frac{1}{\sqrt{B_k}} \cdot \frac{\|\sigma_k\|_1}{\|\nabla \Lcal(x_k)\|_1} \right).
\end{equation}
The structure in $\eta_k^*(B_k)$ reveals the \emph{Manhattan ($\ell_1$) GNS}:
\begin{tcolorbox}[
    enhanced,
    colback=white,       
    colframe=black,      
    coltitle=black,      
    title=\textbf{Manhattan ($\ell_1$) GNS for \signsgd},
    sharp corners,       
    boxrule=0.8pt,       
    bottom=4mm,
    attach boxed title to top left={xshift=0.5cm, yshift*=-0.8\baselineskip},
    boxed title style={
        frame hidden,    
        colback=white,   
        left=1mm, right=1mm 
    }
]
\begin{equation} \label{eq:l1_GNS}
    \Bcal_{\ell_1}(x_k) = \frac{\|\sigma_k\|_1^2}{\left\|\nabla\Lcal(x_{k})\right\|_1^2}.
\end{equation}
\end{tcolorbox}

The ${\ell_1}$ GNS fundamentally differs from the Euclidean ${\ell_2}$ GNS in measuring the gradient norm via the $\ell_1$ norm and noise via the squared sum of standard deviations $\|\sigma_k\|_1^2$.
Geometrically, ${\ell_1}$ GNS controls the probability of correct sign assignment for coordinate-wise updates.

While standard \signsgd is known to converge only to a neighborhood (determined by $\|\sigma_k\|_1$) when using a fixed batch size and learning rate \citep{karimireddy2019error}, adhering to this adaptive schedule ensures the noise term in \Cref{eq:signSGD_expected_loss_improv} remains controlled. This effectively behaves as a variance reduction mechanism similar to error feedback or momentum in \signum \citep{bernstein2018signsgd} (see Appendix~\ref{sec:theory}).
Substituting $B_k=\theta^{-2}\Bcal_{\ell_1}(x_k)$ into \eqref{eq:signSGD_expected_loss_improv} yields the factor $(1 - \theta)^2$, which appears explicitly in our convergence results (see \cref{th:convergence_sign_CBS}).


\subsection{Stochastic Spectral Descent (\specgd)}

Similar to \signsgd, \specgd for matrix weights $X_k \in \Rmbb^{m \times n}$ (assuming $m \leq n$) defines the search direction $P_k = \matsign(G_k)$, which corresponds to the steepest descent direction under the Schatten-$\infty$ (spectral) norm. 
The matrix sign is given by $P_k = U_k V_k^\top$, where $G_k = U_k \Sigma_k V_k^\top$ is the reduced SVD of the gradient. 
Let $C_{\text{row}, k} \in \Smbb^m_+$ be the row-wise aggregated covariance matrix\footnote{The column-wise aggregated covariance matrix $C_{\text{col}, k} = \Embb_k[(\nabla \lcal(X_k; \xi) - \nabla \mathcal{L}(X_k))^\top (\nabla \lcal(X_k; \xi) - \nabla \mathcal{L}(X_k))]$ yields an equivalent nuclear norm  $\|C_{\text{row}, k}^{1/2}\|_{\Scal_1} = \tr(C_{\text{row}, k}^{1/2}) = \tr(C_{\text{col}, k}^{1/2}) = \|C_{\text{col}, k}^{1/2}\|_{\Scal_1}$. Therefore, if $m > n$, we can use $C_{\text{col}, k} \in \Smbb_+^n$ instead to save memory.}, i.e.,
\begin{equation*}
C_{\text{row}, k} =
\Embb_k \left[
\begin{aligned}
&\big(\nabla \lcal(X_k; \xi) - \nabla \mathcal{L}(X_k) \big) \\
&\quad \cdot
\big(\nabla \lcal(X_k; \xi) - \nabla \mathcal{L}(X_k) \big)^\top
\end{aligned}
\right].
\end{equation*}

The analysis is strictly analogous to the \signsgd setting, replacing vector $\ell_1$ norms with nuclear (Schatten-1) norms. 
Specifically, we use a similar error bound:
\begin{equation}
    \label{eq:specGD_error_bound}
    \Embb_k\left[\|\nabla \Lcal(X_k) - G_k\|_{\Scal_1} \right] 
    \leq \frac{\|C_{\text{row}, k}^{1/2}\|_{\Scal_1}}{\sqrt{B_k}}.
\end{equation}
Here, the nuclear norm $\|C_{\text{row}, k}^{1/2}\|_{\Scal_1}$ represents the sum of the singular values of the row-wise covariance matrix.

Substituting \cref{eq:specGD_error_bound} into \Cref{eq:non_euclidean_improvement_initial_bound}, and noting that the inner product of the search direction is $\langle P_k, P_k \rangle = \tr(V_k U_k^\top U_k V_k^\top) = \tr(I_r) = r \leq m$, where $\rank(G_k) = r$, we obtain:
\begin{equation}
\label{eq:spectral_expected_loss_improv}
    \begin{split}
        & \Delta \Lcal (X_k, \eta_k, \|\cdot\|_{\Scal_\infty}, B_k) \\
        & \quad \geq \eta_k \left(\|\nabla\Lcal(X_{k})\|_{\Scal_1} - \frac{\|C_{\text{row}, k}^{1/2}\|_{\Scal_1}}{\sqrt{B_k}} \right) - \frac{r \eta_k^2}{2}.
    \end{split}
\end{equation} 
Maximizing this bound yields the optimal learning rate and expected improvement:
\begin{align}
    \label{eq:spec_optimal_lr}
    \eta_k^*(B_k) & = \frac{\left\|\nabla\Lcal(x_{k})\right\|_{\Scal_1}}{r} \left(1 - \frac{1}{\sqrt{B_k}} \cdot \frac{\|C_{\text{row}, k}^{1/2}\|_{\Scal_1}}{\|\nabla \Lcal(x_k)\|_{\Scal_1}} \right).
\end{align}
This admits the Nuclear ($\Scal_1$) GNS, which mimics the $\ell_1$ GNS in \Cref{eq:l1_GNS}:
\begin{tcolorbox}[
    enhanced,
    colback=white,       
    colframe=black,      
    coltitle=black,      
    title=\textbf{Nuclear ($\Scal_1$) GNS for \specgd},
    sharp corners,       
    boxrule=0.8pt,       
    bottom=4mm,
    attach boxed title to top left={xshift=0.5cm, yshift*=-0.8\baselineskip},
    boxed title style={
        frame hidden,    
        colback=white,   
        left=1mm, right=1mm 
    }
]
\begin{equation} \label{eq:nuclear_GNS}
    \Bcal_{\Scal_1}(x_k) = \frac{ \|C_{\text{row}, k}^{1/2}\|_{\Scal_1}^2}{\|\nabla\Lcal(X_k)\|_{\Scal_1}^2}.
\end{equation}
\end{tcolorbox}

\paragraph{Analogous properties.} 
Most of the geometric intuition and convergence guarantees derived for \signsgd transfer to \specgd, substituting $\lcal_\infty/\lcal_1$ norms with their spectral/nuclear counterparts. 
When $B_k \ll \Bcal_{\Scal_1}$, the noise dominates the gradient's spectral structure, causing the unitary factor $U_k V_k^\top$ to point in a random direction effectively uncorrelated with the true spectral geometry.
Ensuring $B_k \gg \Bcal_{\Scal_1}$ aligns the search direction with the true gradient's dominant singular subspaces.
Consequently, the same adaptive schedule introduces the identical $(1 - \theta)^2$ factor into our convergence results (see \Cref{th:convergence_specGD_CBS}).

Although we use normalized steepest descent for \signsgd and \specgd, the dual norm scaling ($\|\nabla \mathcal{L}(x_k)\|_1$ and $\|\nabla \mathcal{L}(X_k)\|_{\mathcal{S}_1}$) still appears in the optimal learning rates (\Cref{eq:sign_optimal_lr,eq:spec_optimal_lr}) and are essential for proving convergence (see Appendix~\ref{sec:theory}). 
These expressions also omit the influence of the smoothness constant, which affects the learning rate, as addressed in the convergence analysis.

    
    

\begin{table}
    \caption{
    Gradient Noise Scale (GNS) for \sgd, \signsgd, and \specgd expressed as a squared ratio of dual norms.
    Note the consistent pattern: primal norm $\|\cdot\|$, dual norm $\|\cdot\|_*$, GNS $\propto (\|\text{Noise}\|_* / \|\text{Signal}\|_*)^2$.
    }
    \centering
    \small
    \setlength{\tabcolsep}{6pt}
    \begin{tabular}{lcc}
    \toprule
    \textbf{Method} &
    \textbf{Geometry} &
    \textbf{GNS} \\
    \midrule
    
    \sgd &
    $\ell_2$ &
    $\displaystyle 
    \frac{\tr(C_k)}
         {\|\nabla\Lcal(x_k)\|_{2}^{2}}
    $ 
    \\
    \midrule
    
    \signsgd &
    $\ell_{\infty}$&
    $\displaystyle 
        \frac{\|\sigma_k\|_{1}^2}
             {\|\nabla\Lcal(x_k)\|_{1}^2}
    $
    \\
    \midrule
    \specgd &
    $\Scal_\infty$ &
    $\displaystyle 
       \frac{
          \|C_{\text{row}, k}^{1/2}\|_{\Scal_1}^2
       }{
          \|\nabla\Lcal(X_k)\|_{\Scal_1}^2
       }
    $
    \\
    \bottomrule
    \end{tabular}
    \label{table:gns-norm-geometry}
\end{table}

\subsection{Examining Definitions of Critical Batch Size} \label{subsec:cbs_derivations}


To formally link the GNS to CBS, we consolidate the optimal expected single-step improvements for each method by plugging their optimal learning rates (\Cref{eq:sgd_optimal_lr,eq:sign_optimal_lr,eq:spec_optimal_lr}) into their respective expected improvement bounds (\Cref{eq:SGD_expected_loss_improv,eq:signSGD_expected_loss_improv,eq:spectral_expected_loss_improv}):
\begin{align}
    \Delta_k^*(B_k, \|\cdot\|_2)        &= \tfrac{\|\nabla\Lcal(x_{k})\|_2^2}{2}\cdot\left(1 + \tfrac{\Bcal_{\ell_2}(x_k)}{B_k}\right)^{-1}, \nonumber \\
    \Delta_k^*(B_k, \|\cdot\|_\infty)   &= \tfrac{\|\nabla\Lcal(x_k)\|_1^2}{2d}\left(1 - \sqrt{\tfrac{\Bcal_{\lcal_1}(x_k)}{B_k}}\right)^2, \label{eq:delta_star_summary}\\
    \Delta_k^*(B_k, \|\cdot\|_{\Scal_\infty})   &= \tfrac{\|\nabla\Lcal(X_k)\|_{\Scal_1}^2}{2r}\left(1 - \sqrt{\tfrac{\Bcal_{\Scal_1}(x_k)}{B_k}}\right)^2. \nonumber
\end{align}
These quantities increase monotonically with the batch size before eventually saturating, consistent with the empirical behavior observed in \citet{shallue2019measuring}. 


The CBS was originally defined as the ``turning point" in batch size beyond which further increases yield diminishing returns in expected improvement. 
Following \citet{mccandlish2018empirical}, we define this as the batch size $B_\CBS$ required to achieve a specific fraction $\kappa \in (0, 1)$ of the maximum possible deterministic improvement $\Delta_k^*(\infty, \|\cdot\|)$:
\begin{equation*}
    B_{\CBS, k} (\kappa, \|\cdot\|) = \inf \{B_k: \Delta_k^* (B_k, \|\cdot\|) \geq \kappa \Delta_k^* (\infty, \|\cdot\|)\}.
\end{equation*}
Solving this inequality for the specific forms in \Cref{eq:delta_star_summary} yields:
\begin{align}
    B_{\CBS, k} (\kappa, \|\cdot\|_2) 
    &= \frac{\kappa}{1-\kappa}\Bcal_{\lcal_2}(x_k),    \nonumber\\
    B_{\CBS, k} (\kappa, \|\cdot\|_\infty) 
    &= \left(\frac{1}{1-\sqrt{\kappa}}\right)^2\Bcal_{\lcal_1}(x_k), \label{eq:CBS_general_forms}\\
    B_{\CBS, k} (\kappa, \|\cdot\|_{\Scal_\infty}) 
    &= \left(\frac{1}{1-\sqrt{\kappa}}\right)^2\Bcal_{\Scal_1}(x_k). \nonumber
\end{align}
Irrespective of the specific choice of $\kappa$, the CBS is always a linear scaling of the respective GNS. 

\paragraph{Turning point definitions.}

The constant $\kappa$ defines where on the saturation curve the ``point of diminishing returns" occurs. 
Conceptually, this threshold separates two regimes: when $B_k \ll \Bcal(x_k)$, optimization is noise-dominated and scaling the batch size yields a near-linear reduction in training steps; when $B_k \gg \Bcal(x_k)$, the gradient error saturates, rendering additional samples computationally redundant.

While $\kappa = 1/2$ is standard for SGD \citep{mccandlish2018empirical}, the saturation curves for non-Euclidean norms exhibit different geometric properties. 
As seen in \Cref{eq:delta_star_summary}, the SGD improvement curve is strictly concave, whereas the curves for $\ell_\infty$ and $\Scal_1$ are initially convex, then concave. 
This implies that for non-Euclidean optimizers, there exists an inflection point where the rate of marginal gain peaks, motivating alternative definitions of $\kappa$ (see Appendix~\ref{appendix:alt-cbs}).
However, since our derivation relies on lower bounds, the theoretical constant relating CBS to GNS is loose.
Therefore, in practice, we treat the GNS scaling as a tunable hyperparameter $\theta$ to empirically traverse the Pareto frontier between sample efficiency and steps.

\section{Implementation} \label{sec:implementation}

To translate our theoretical findings into a practical system, we introduce a variance estimation framework designed for distributed training.
Our approach treats the local mini-batch gradients on each data-parallel rank as independent samples from the gradient distribution, enabling us to estimate population statistics without expensive auxiliary passes. 
This methodology extends the Euclidean variance estimation techniques from \citet{lau2024adaptive,lau2024adadagrad} to non-Euclidean GNS metrics required for \signsgd and \specgd.

\subsection{Noise Estimation}
\label{sec:noise_estimation}

In a distributed data-parallel (DDP) setup with $R$ ranks, the global batch $\Scal_k$ of size $B_k$ is partitioned into disjoint local batches $\Scal_k^j$ of size $\frac{B_k}{R}$. 
Each rank $j$ computes a local mini-batch gradient $g_k^j= \tfrac{R}{B_k}\sum_{\xi \in \Scal_k^j} \nabla \lcal (x_k, \xi)$. These local updates are aggregated to form the global gradient $g_k = \frac{1}{R} \sum_{j = 1}^{R} g_k^j$ via an \texttt{AllReduce} operation.\footnote{In Fully-Sharded Data-Parallel (FSDP), the \texttt{AllReduce} is decomposed into \texttt{AllGather} and \texttt{ReduceScatter} primitives. Our variance computation hooks are inserted prior to the \texttt{ReduceScatter} to access the unsharded local gradients.} 

We leverage the local mini-batch gradients $g_k^j$ to construct an unbiased estimator for the gradient noise. 
The estimators for the coordinate-wise variance $\hat{\sigma}_k$ and the row-wise covariance matrix $\hat{C}_{\text{row}, k}$ are:
\begin{align} 
    \left[\hat{\sigma}_k\right]_i^2 & = \frac{B_k}{R - 1}\left(\frac{1}{R} \sum_{j = 1}^R [g_k^j]_i^2 - \left[g_k\right]_i^2\right), 
    \label{eq:GNS_estimate_SGD_signSGD}\\
    \hat{C}_{\text{row}, k} & = \frac{B_k}{R-1}\left(\frac{1}{R}\sum_{j=1}^{R} G_k^j (G_k^j)^\top - G_k G_k^\top\right).
    \label{eq:spec_sample_variance_estimate}
\end{align}
The scaling factor $\frac{B_k}{R - 1}$ applies two necessary corrections simultaneously: the Bessel correction $\frac{R}{R-1}$ converts the sample variance of the ranks into an unbiased population estimate, and the factor $\frac{B_k}{R}$ rescales the variance of the mini-batches to the variance of a single sample.

Using the global gradient ($g_k$ or $G_k$) as the signal estimate (following \citet{bollapragada2018adaptive}), we obtain the final computable GNS metrics:
\begin{equation} \label{eq:GNS_estimates}
    \hat{\Bcal}_{\ell_1} = \frac{\|\hat{\sigma}_k\|_1^2}{\|g_k\|_1^2}
    \quad \text{and} \quad
    \hat{\Bcal}_{\Scal_1} = \frac{\|\hat{C}_{\text{row}, k}^{1/2}\|_{\Scal_1}^2}{\|G_k\|_{\Scal_1}^2}.
\end{equation}
All estimations are performed prior to gradient clipping.

\paragraph{Implementation details.} 
For the $\ell_1$ GNS, each rank computes the element-wise square $(g_k^j)^2$ locally.
In standard DDP, this requires storing one additional gradient copy per rank and performing an additional \texttt{AllReduce} to average the squares. 
In FSDP, this memory overhead is mitigated by instead performing a \texttt{ReduceScatter} on the squared gradients, effectively sharding the variance statistics consistent with the parameters.

For the $\Scal_1$ GNS, each rank computes the local Gram matrix $G_k^j (G_k^j)^\top \in \Rmbb^{m \times m}$.
While FSDP can similarly \texttt{ReduceScatter} this term, computing the final nuclear norm $\|\hat{C}_{\text{row}, k}^{1/2}\|_{\Scal_1}$ requires a subsequent \texttt{AllGather} to reconstruct the full covariance matrix on each rank.
Because these computations are restricted to element-wise or block-wise statistics, the metric can be computed independently for each model component, ensuring full compatibility with tensor and pipeline parallelism. We provide a detailed wall-clock, communication, and memory cost analysis of the $\ell_1$ and $\Scal_1$ GNS estimators in Appendix~\ref{app:cost_analysis}.

\subsection{Adaptive Batch Size Heuristics}

We implement several heuristics to stabilize our adaptive strategy (\cref{alg:structure}). First, we apply \emph{exponential moving averages} to the noise ($N_k$) and signal ($M_k$) components to mitigate the impact of stochastic fluctuations.
For instance, for \signsgd:
\begin{equation} \label{eq:var_ema_method}
    \begin{split}
        N_{k} &= \beta^{N} N_{k-1} + (1 - \beta^N)\|\hat{\sigma}_k\|_1^2, \\
        M_{k} &= \beta^{M} M_{k-1} + (1 - \beta^M)\|g_k\|_1^2,
    \end{split}
\end{equation}
yielding the batch size $B_k = N_{k} / (\theta^2 M_{k})$. 
Second, the noise ($N_k$) and signal ($M_k$) components are \emph{updated periodically} to reduce computational overhead \cite{belias2025one}.
Batch size increases are forced to occur after an initial warmup and grow monotonically to ensure scaling aligns with an improving signal-to-noise ratio while preventing reversals caused by transient noise.
Third, to maximize the benefits of lower-variance gradient estimates, we \emph{scale the learning rate proportional to the square root of the batch size} ($\eta_k \propto \sqrt{B_k}$) \cite{merrill2025critical}.
Finally, for composite optimizers that utilize \specgd and \muon for hidden layers and \adamw for 1D parameters, we compute the GNS solely from the 2D parameter groups, as these dominate the model’s total parameter count.

\begin{algorithm}[tb]
  \caption{Adaptive Batch Sizes (Non-Euclidean GNS)}
  \label{alg:structure}
  \begin{algorithmic}[1]
    \STATE {\bfseries Input:} $\beta^N, \beta^M \in [0, 1)$; update frequency $F$; initial warm-up $I$; total steps $K$; $\theta \in (0, 1)$; learning rate sequence $\{\eta_k\}$; learning rate scaling $\omega_0 = 1$.
    \FOR{$k = 0, ..., K - 1$}
        \STATE Compute $g_k$ and estimate $\hat{\sigma}_k$ using \eqref{eq:GNS_estimate_SGD_signSGD} or $\hat{C}_{\text{row}, k}^{1/2}$ using \eqref{eq:spec_sample_variance_estimate}.
        \STATE Compute $p_k = \arg\max_{\|p\| \leq 1} \langle g_k, p \rangle$.
        \STATE Update $x_{k+1} = x_k - \omega_k \eta_k p_k$.
        \IF{$k \bmod F = 0$}
            \STATE Update $N_k$ and $M_k$ using \eqref{eq:var_ema_method}.
            \IF{$k \geq I$}
                \STATE Update $B_{k+1} = \max\left\{\left\lceil \frac{N_k}{\theta^2 M_k}\right\rceil, B_k\right\}$.
                \STATE Update $\omega_{k + 1} = \omega_k \sqrt{\frac{B_{k+1}}{B_k}}$.
            \ELSE
                \STATE Update $B_{k + 1} = B_k$ and $\omega_{k + 1} = \omega_k$.
            \ENDIF
        \ENDIF
    \ENDFOR
  \end{algorithmic}
\end{algorithm}

\section{Experiments} \label{sec:exps}
\begin{table*}[ht]

\caption{Validation loss and steps reduction (\%) for the 160M Llama 3 model trained for 3.2B tokens (10 seeds) over the C4 dataset. The adaptive batch size methods start from the optimal constant baseline ($B = 256$ for \signsgd and $B=64$ for all others). The steps reduction (\%) represents the median percent reduction in steps to reach the baseline's minimum validation loss by Non-Euclidean GNS.}

\label{tab:Llama160M_summary}

\begin{center}
\begin{small}
\begin{sc}
\begin{tabular}{lccccc}
    \toprule
    \multirow{2}{*}[-0.75ex]{Optimizer} & \multicolumn{4}{c}{validation loss} & \multirow{2}{*}[-0.75ex]{\begin{tabular}[c]{@{}c@{}}Steps \\ Reduction(\%)\end{tabular}} \\ \cmidrule{2-5}
     & B = 64 & B = 256 & Euclidean GNS & Non-Euclidean GNS & \\
    \midrule
    \rowcolor{SignSGDColor}
    \signsgd & $4.1390 \pm 0.0235$ & $3.9357 \pm 0.0057$ & $3.8906 \pm 0.0117$ &$3.8396 \pm 0.0012$ &  22.58 \\ 
    \rowcolor{SignSGDColor}
    \signum & $3.3737 \pm 0.0045$ & $3.4460 \pm 0.0075$ & $3.3842 \pm 0.0055$ & $3.3707 \pm 0.0026$  & 66.61 \\ 
    \rowcolor{SignSGDColor}
    \adamw & $3.2991 \pm 0.0054$ & $3.3228 \pm 0.0101$ & $3.3127 \pm 0.0101$& $3.3031 \pm 0.0054$ &  67.13 \\ 
    \rowcolor{SpecSGDColor}
    \specgd & $3.7489 \pm 0.0096$ & $3.9404 \pm 0.0074$ & $3.7436 \pm 0.0274$ & $3.7285 \pm 0.0094$ &  16.49 \\ 
    \rowcolor{SpecSGDColor}
    \muon & $3.3041 \pm 0.0021$ & $3.3411 \pm 0.0019$ & $3.3181 \pm 0.0072$ & $3.3061 \pm 0.0027$ &  66.77 \\
    \bottomrule
\end{tabular}
\end{sc}
\end{small}
\end{center}
\vskip -0.1in
\end{table*}

To evaluate the empirical effectiveness of our methodology, we conduct language and vision training experiments across a diverse set of optimizers, including \signsgd, \signum, \adamw, \specgd and \muon. We compare our proposed adaptive strategy detailed in \cref{alg:structure} against constant small and large batch size baselines. 
See Appendix~\ref{app:exp_setup} for further details regarding our experimental setup.

\subsection{Language Workloads} \label{subsec:language_models}



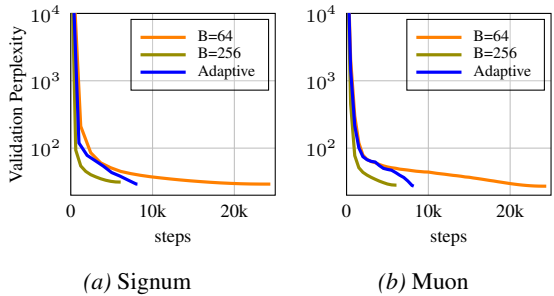
\begin{figure}[h]
    \centering
    \begin{subfigure}{0.45\columnwidth}
        \centering
        \begin{tikzpicture}
            \begin{groupplot}[
                Llama3_perplexity_plots,     
                legend pos=north east,
                legend columns=1,               
                legend style={
                    font=\tiny,           
                    row sep=-3pt,               
                    fill opacity=0.2,           
                    draw opacity=1, 
                    text opacity=1,
                    column sep=2pt,
                    cells={anchor=west}         
                },
                legend entries={B=64, B=256, Adaptive},
            ]
                \nextgroupplot[
                    xlabel={steps},
                    ylabel={Validation Perplexity},
                    xmin=0, xmax=25000,
                    xtick={0, 1e4, 2e4},
                    xticklabels={0, 10k, 20k}, 
                ]
                \addplot table [x=_step, y=perplexity, col sep=comma] {Llama_results/llama_160M_main_csv/signum_64.csv};
                \addplot table [x=_step, y=perplexity, col sep=comma] {Llama_results/llama_160M_main_csv/signum_256.csv};
                \addplot table [x=_step, y=perplexity, col sep=comma] {Llama_results/llama_160M_main_csv/signum_gns.csv};
            \end{groupplot}
        \end{tikzpicture}
        \caption{\signum}
        \label{fig:llama160M_signum_validation_perplexity}
    \end{subfigure}
    \begin{subfigure}{0.45\columnwidth}
        \centering
        \begin{tikzpicture}
            \begin{groupplot}[
                Llama3_perplexity_plots,
                legend pos=north east,
                legend columns=1,               
                legend style={
                    font=\tiny,           
                    row sep=-3pt,               
                    fill opacity=0.2,           
                    draw opacity=1, 
                    text opacity=1,
                    column sep=2pt,
                    cells={anchor=west}         
                },
                legend entries={B=64, B=256, Adaptive},
            ]
                \nextgroupplot[
                    xlabel={steps}, 
                    xmin=0, xmax=25000,
                    xtick={0, 1e4, 2e4},
                    xticklabels={0, 10k, 20k}, 
                ]            
                \addplot table [x=_step, y=perplexity, col sep=comma] {Llama_results/llama_160M_main_csv/MUON_64.csv};
                \addplot table [x=_step, y=perplexity, col sep=comma] {Llama_results/llama_160M_main_csv/MUON_256.csv};
                \addplot table [x=_step, y=perplexity, col sep=comma] {Llama_results/llama_160M_main_csv/MUON_gns.csv};
            \end{groupplot}
        \end{tikzpicture}
        \caption{\muon}
        \label{fig:llama160M_MUON_validation_perplexity}
    \end{subfigure}
    \vspace{0.1cm} \\
    \caption{Comparison of constant batch sizes ($B=64$ and $B=256$) and our adaptive batch size method for the 160M Llama 3 model. For both \signum (Fig \ref{fig:llama160M_signum_validation_perplexity}) and \muon (Fig \ref{fig:llama160M_MUON_validation_perplexity}), our adaptive strategy matches the final perplexity of the smaller batch size baseline while significantly reducing the total steps.}
    \label{fig:llama160M_signum_and_muon_validation_perplexity}
\end{figure}

We train 160M and 1B parameter Llama 3 models \citep{grattafiori2024llama} on the C4 dataset \citep{raffel2020exploring} using Chinchilla-optimal \citep{hoffmann2022training} budgets of 3.2B and 22B tokens, respectively. Both configurations used a 2048 sequence length, trained via DDP on a single node with eight H100 GPUs (97GB memory). Our experimental results for the 160M and 1B models are summarized in \Cref{tab:Llama160M_summary,tab:Llama1B_summary}, respectively. 

As illustrated in \cref{fig:llama160M_signum_and_muon_validation_perplexity}, our proposed adaptive batch size strategy reduces the total steps for \signum and \muon while matching the validation perplexity of the constant small-batch baseline. 
Both \signsgd and \specgd are able to achieve a significantly better loss than their constant batch size baselines.
For the 160M Llama 3 model, the adaptive approach achieves up to a 66.77\% reduction in steps required to reach the baseline minimum validation loss for Muon. 
For the 1B Llama 3 model, \signsgd and \signum similarly demonstrate substantial efficiency gains of 31.84\% and 12.11\%, respectively, while \adamw does not achieve the baseline validation loss under the adaptive strategy. 
We provide additional ablation studies in Appendix \ref{app:languange}.



\begin{table*}[t]
\begin{center}
\caption{Validation loss and steps reduction (\%) for the 1B Llama 3 model trained for 22B tokens (10 seeds) over the C4 dataset. The adaptive batch size method starts at $B = 64$, compared against the $B = 256$ constant-batch baseline. The steps reduction (\%) indicates the percent reduction in steps to reach the baseline's validation loss.}
\label{tab:Llama1B_summary}
\small 
\begin{sc}
\begin{tabular}{lccc}
\toprule
\multirow{2}{*}[-0.75ex]{Optimizer} & \multicolumn{2}{c}{Validation Loss} & \multirow{2}{*}[-0.75ex]{\begin{tabular}[c]{@{}c@{}}Steps \\ Reduction (\%)\end{tabular}} \\ \cmidrule{2-3}
 & B = 256 & Non-Euclidean GNS &  \\
\midrule
\rowcolor{SignSGDColor}
\signsgd & $3.1417 \pm 0.0047$ & $2.9946 \pm 0.0201$ & 31.84 \\ 
\rowcolor{SignSGDColor}
\signum & $2.8306 \pm 0.0051$ & $2.8354 \pm 0.0073$ & 12.11 \\ 
\rowcolor{SignSGDColor}
\adamw & $2.7701 \pm 0.0037$ & $2.8023 \pm 0.0049$ & - \\
\bottomrule
\end{tabular}
\end{sc}
\end{center}
\end{table*}


\subsection{Vision Workloads}  \label{subsec:vision_models}

\begin{table*}[tb]
\begin{center}
\caption{Validation loss and steps reduction (\%) for the SimpleViT model trained (5 seeds) over ImageWoof dataset. The adaptive batch method starts at $B = 128$, compared against the $B = 128$ constant-batch baseline. The steps reduction (\%) indicates the percent reduction in steps to reach the baseline's minimum validation loss. The $B = 512$ column is reported for reference. The Adaptive column uses the GNS measured in the dual norm matched to each optimizer's geometry ($\ell_2$ for \sgd/\msgd/\nsgd, $\ell_1$ for \signsgd/\signum/\adamw, and $\Scal_1$ for \muon).}

\label{tab:Imagewoof_summary}
\begin{small}
\begin{sc}
\begin{tabular}{lcccc}
\toprule
\multirow{2}{*}[-0.75ex]{Optimizer} & \multicolumn{3}{c}{validation loss} & \multirow{2}{*}[-0.75ex]{\begin{tabular}[c]{@{}c@{}}Steps \\ Reduction(\%)\end{tabular}} \\ \cmidrule{2-4}
 & B = 128 & B = 512 & Adaptive &  \\
\midrule
\rowcolor{SGDColor}
\msgd & $1.1966 \pm 0.0129$ & $1.4539 \pm 0.0118$ & $1.1849 \pm 0.0044$  &  4.14
\\ 
\rowcolor{SGDColor}
\sgd & $1.4829 \pm 0.0187$ & $1.6473 \pm 0.0076$ & $1.3179 \pm0.0239$ & 40.29
\\ 
\rowcolor{SGDColor}
\nsgd & $1.1979 \pm 0.0087$ & $1.5610 \pm 0.0065$ & $1.3286 \pm 0.0028$ & -
\\ 
\rowcolor{SignSGDColor}
\signsgd & $1.2003 \pm 0.0087$ & $1.5153 \pm 0.0080$ & $1.1927 \pm0.0064$ & 27.44 \\ 
\rowcolor{SignSGDColor}
\signum & $1.2987 \pm 0.0217$ & $1.5109 \pm 0.0049$ & $1.2322 \pm 0.0029$ & 11.30  \\ 
\rowcolor{SignSGDColor}
\adamw & $0.9750 \pm 0.0151$ & $1.0653 \pm 0.0051$ & $0.9338 \pm 0.0030$ &  10.86 \\ 

\rowcolor{SpecSGDColor}
\muon & $0.6578 \pm0.0218$ & $0.7156 \pm 0.0060$ & $0.6433 \pm0.0057$ &  57.65\\
\bottomrule
\end{tabular}
\end{sc}
\end{small}
\end{center}
\end{table*}

We train a SimpleViT \citep{beyer2022better} on the Imagewoof \citep{Howard_Imagewoof_2019} dataset.
For each evaluated optimizer, we compare a constant batch ($B=128$) baseline against our proposed adaptive batch size strategy, reporting performance from the best configuration found from our hyperparameter sweep.
We also conduct evaluations on other architectures and datasets; see Appendix~\ref{app:exp_setup_vision} and \ref{app:vision}.

As shown in Table \ref{tab:Imagewoof_summary}, for most optimizers, adaptive batching achieves comparable or lower validation loss while reducing the number of optimizer steps.

Overall, our results indicate that using $\ell_1$ and $\Scal_1$ GNS effectively reduces the number of steps on vision tasks without compromising validation performance.

\section{Conclusion} \label{sec:conc}

In this work, we established the $\ell_1$ and $\Scal_1$ GNS as rigorous frameworks for adaptively selecting batch sizes for \signsgd and \specgd.
While this provides a robust foundation for generalized steepest descent, extending these principles to preconditioned and stateful optimizers remains an open challenge. 
For instance, how do we handle exponential moving averages of the gradients or second moments (e.g., in \adamw or \muon)? 
How do we unify different noise scales when different modules utilize different optimizers? 
Additionally, how do we account for more general Hessian structures in our GNS metrics?

Beyond batch size selection, the geometry-aware GNS can serve as a fundamental diagnostic tool for modern training stacks. 
By quantifying noise relative to the optimizer's update direction, this metric offers a precise lens for monitoring training stability, detecting distribution shifts, and guiding decisions in quantization or sparsity. 
We hope these signals motivate the development of next-generation training algorithms that explicitly leverage variance metrics and covariance adaptation.






\newpage

\section*{Impact Statement} \label{sec:impact_statement}
This paper presents work whose goal is to advance the field of machine learning. There are many potential societal consequences of our work, none of which we feel must be specifically highlighted here.

\section*{Acknowledgements}
We thank Raghu Bollapragada, Anna Cai, Jiaming Cui, Aaron Defazio, Runa Eschenhagen, Wei Feng, Chien-Chin Huang, Tsung-Hsien Lee, Gavin Zhang, and Iris Zhang for their discussions on the algorithm, support in experimentation, and detailed feedback on the manuscript. We also thank Adnan Aziz, Maxim Naumov, Sandeep Parab, Joel Pobar, and Chunqiang Tang for their managerial support of this work.

This work was supported by RBC Borealis through the RBC Borealis AI Global Fellowship Award, which was awarded to Hiroki Naganuma.
This research was supported by grants from NVIDIA and utilized NVIDIA A100 GPUs provided through the NVIDIA Academic Grant Program.  
The authors also acknowledge the support of Compute Canada and Mila computing clusters for experimental resources.


\bibliography{example_paper}
\bibliographystyle{icml2026}

\newpage
\appendix
\onecolumn

\section{Theoretical Results} 
\label{sec:theory}

This section provides formal derivations for the theoretical bounds introduced in Section 
\hn{\ref{sec:alg}}
and details the convergence analyses for \signsgd and \specgd under the adaptive batch size strategy.

\begin{lemma} 
    \label[lemma]{lem:app_general_norm_inner_product}
    Let $g_k$ be a gradient estimate such that $\Embb_k\left[g_k\right] = \nabla \Lcal(x_k)$. Let $\|\cdot\|$ be a general norm and $\|\cdot\|_*$ be its associated dual norm. The stochastic steepest descent direction $p_k = \arg\max_{\|p\| \leq 1} \langle g_k, p\rangle$ satisfies:
    \begin{equation*}
        \Embb_k \left[\langle \nabla \Lcal (x_k), p_k \rangle \right] \geq \left\|\nabla \Lcal(x_k)\right\|_* - \Embb_k \left[\left\|\nabla \Lcal (x_k) - g_k\right\|_* \right].
    \end{equation*}
\end{lemma}
\begin{proof}
    Let $\ecal_k = g_k - \nabla \Lcal(x_k) $ be the gradient estimation error at iteration $k$. The inner product $\langle \nabla \Lcal(x_k), p_k \rangle$ can be lower bounded as:
    \begin{align}
        \langle \nabla \Lcal(x_k), p_k \rangle 
        & = \langle g_k - \ecal_k, p_k \rangle \nonumber \\
        & = \langle g_k, p_k \rangle - \langle \ecal_k, p_k \rangle \nonumber \\
        & = \|g_k\|_* - \langle \ecal_k, p_k \rangle \nonumber \\
        & \geq \|g_k\|_* - \|\ecal_k\|_* \label{eq:inner_prod_lower_bound}
    \end{align}
    where the third equality follows from the definition of the dual norm (since $p_k$ is chosen to maximize the inner product with $g_k$), and the final bound follows from H\"{o}lder's Inequality:
    \begin{equation*}
        \langle \ecal_k, p_k \rangle \leq \|\ecal_k\|_* \| p_k \| = \|\ecal_k\|_*
    \end{equation*}
    (noting that $\|p_k\| \leq 1$).

    Taking the expectation of the inequality \eqref{eq:inner_prod_lower_bound} 
    given $x_k$, we apply Jensen's inequality ($\Embb_k\left[\|g_k\|_* \right] \geq \left\|\Embb_k\left[g_k\right]\right\|_* = \left\|\nabla \Lcal(x_k)\right\|_*$) to the first term to complete the proof. 
    
\end{proof}

\begin{lemma} \label[lemma]{lem:app_signSGD_error_bound}
Let $g_k = \tfrac{1}{B_k}\sum_{\xi \in \Scal_k} \nabla \lcal(x_k; \xi) \in \Rmbb^d$ be a gradient estimate computed over a mini-batch $\Scal_k$ of $B_k$ i.i.d samples drawn from $\Dcal$ at iteration $k$. Then the expected $\ell_1$ norm of the estimation error can be bounded as:
\begin{equation*}
    \Embb_k\left[\|\nabla \Lcal(x_k) - g_k\|_1 \right] \leq \frac{\|\sigma_k\|_1}{\sqrt{B_k}},
\end{equation*}
where $\sigma_k$ is the component-wise standard deviation for the single sample gradient $\nabla \ell(x_k; \xi)$ given $x_k$, i.e., for each component $i = 1, ..., d$:
\begin{equation*}
    [\sigma_k]_i^2 = \Embb \left[([\nabla \ell(x_k; \xi) - \nabla \Lcal(x_k)]_i)^2 \right].
\end{equation*}
\end{lemma}
\begin{proof}
Let $[x]_i$ denote the $i$-th element of the vector $x \in \Rmbb^d$. Then:
\begin{align*}
    \Embb_k\left[\|\nabla \Lcal(x_k) - g_k\|_1\right]
    &= \sum_{i=1}^{d} \Embb_k\left[\left|\left[\nabla \Lcal(x_k) - g_k\right]_i \right|\right] \\
    &\leq \sum_{i=1}^{d} \sqrt{\Embb_k\left[\left(\left[\nabla \Lcal(x_k) - g_k\right]_i \right)^2 \right] },
\end{align*}
where the final inequality follows from Jensen's inequality.
Since $g_k$ is the average of $B_k$ i.i.d. samples,  each with variance $\left[\sigma_k\right]_i^2$, we obtain:
\begin{align*}
    \Embb_k\left[\|\nabla \Lcal(x_k) - g_k\|_1 \right]
    &\leq \sum_{i=1}^{d} \sqrt{\frac{\left[\sigma_k\right]_i^2}{B_k} } = \frac{\|\sigma_k\|_1}{\sqrt{B_k}},
\end{align*}
completing the proof.

\end{proof}

\begin{lemma} \label[lemma]{lem:app_specGD_error_bound}
Let $G_k = \tfrac{1}{B_k}\sum_{\xi \in \Scal_k} \nabla \lcal(X_k; \xi)$ be a gradient estimate computed over a mini-batch $\Scal_k$ of $B_k$ i.i.d samples drawn from $\Dcal$ at iteration $k$. Then the expected nuclear norm of the estimation error can be bounded as:
\begin{equation*}
    \Embb_k\left[\|\nabla \Lcal(X_k) - G_k\|_{\Scal_1} \right] 
    \leq \frac{\|C_{\text{row}, k}^{1/2}\|_{\Scal_1}}{\sqrt{B_k}},
\end{equation*}
where $C_{\text{row}, k} = \Embb_k\left[ \left(\nabla \Lcal(X_k) - \nabla \lcal(X_k; \xi)\right)\left(\nabla \Lcal(X_k) - \nabla \lcal(X_k; \xi)\right)^\top \right]$ is the row-wise aggregated covariance matrix.
\end{lemma}
\begin{proof}
Let $\Ecal_k = \nabla \Lcal(X_k) - G_k \in \Rmbb^{m \times n}$ be the gradient estimate error. Then:
\begin{equation*}
    \Embb_k\left[\|\mathcal{E}_k\|_{\Scal_1} \right] 
    = \Embb_k\left[\tr\left((\Ecal_k\Ecal_k^\top)^{1 / 2}\right) \right]
    \leq \tr\left(\Embb_k\left[\Ecal_k\Ecal_k^\top \right]^{1 / 2}\right) = \frac{\tr\left(C_{\text{row}, k}^{1/2}\right)}{\sqrt{B_k}},
\end{equation*}
where the inequality follows from Jensen’s inequality together with the concavity of the matrix square root over positive semi-definite matrices.

\end{proof}

\begin{theorem} \label{th:convergence_sign_CBS} 
Assume the objective function $\Lcal(x) = \Embb \left[\lcal(x; \xi)\right] : \Rmbb^d \rightarrow \Rmbb$ is $\ell_{\infty}$ smooth, i.e.,
\begin{equation*}
    \left\|\nabla \Lcal(x) - \nabla \Lcal(x')\right\|_{1} \leq L_{\infty} \|x - x'\|_{\infty} \quad \forall x, x' \in \Rmbb^d,
\end{equation*}
with $L_{\infty}>0$ and bounded below with optimal value $\Lcal^*$. Let $\{x_k\}$ be the sequence generated by the update rule $x_{k+1} = x_k - \eta_k \sign(g_k)$ for $k \geq 0$, where $g_k = \tfrac{1}{B_k}\sum_{\xi \in \Scal_k} \nabla \lcal(x_k; \xi)$ is a gradient estimate computed over a mini-batch of $B_k = \frac{\|\sigma_k\|_1^2}{\theta^2\|\nabla \Lcal (x_k)\|_1^2}$ i.i.d samples with $\theta \in \left(0, 1\right)$ and $\sigma_k$ the component-wise standard deviation for the single sample gradient $\nabla \ell(x_k; \xi)$ given $x_k$.
\begin{enumerate}
    \item If the learning rate is chosen as $\eta_k = \eta = \frac{1}{\sqrt{L_{\infty} K}}$ for $K \geq 1$,
    \begin{align}
        \Embb\left[ \frac{1}{K}\sum_{k=0}^{K-1} \left\|\nabla\Lcal(x_{k})\right\|_1\right]\leq \frac{\sqrt{L_{\infty}}}{(1 -  \theta)\sqrt{K}}\left[\Lcal(x_0) - \Lcal^* + \frac{1}{2}\right].
        \label{eq:app_signsgd_proof_result_1}
    \end{align}

    \item If the learning rate is chosen such that $\sum_{k=0}^\infty \eta_k = \infty$ and $\sum_{k=0}^\infty \eta_k^2 < \infty$,
    \begin{align}
        \liminf_{k \rightarrow \infty} ~ \Embb[\|\nabla \Lcal(x_k)\|_1] = 0. \label{eq:app_signsgd_proof_result_2}
    \end{align}

    \item If the learning rate is chosen as $\eta_k = (1 - \theta)\frac{\|\nabla \Lcal (x_k)\|_1}{L_\infty}$, then for any for $K \geq 1$,
    \begin{align}
        \Embb\left[ \frac{1}{K}\sum_{k=0}^{K-1} \left\|\nabla\Lcal(x_{k})\right\|_1^2\right]\leq \frac{2L_{\infty}}{(1 -  \theta)^2K}(\Lcal(x_0) - \Lcal^*). \label{eq:app_signsgd_proof_result_gns_lr}
    \end{align}
    
    If the objective function is also $\mu_\infty-$strongly convex, i.e.,
    \begin{align*}
        2 \mu_\infty \left(\Lcal (x)  - \Lcal^*\right)\leq \left\|\nabla \Lcal (x)\right\|_1^2 \quad \forall x \in \Rmbb^d,
    \end{align*}
    with $\mu_\infty > 0$, then
    \begin{align}
        \Embb\left[\Lcal(x_{k+1})\right] - \Lcal^*
        \leq  
        \left[1 - \left(1 - \theta\right)^2\frac{\mu_\infty}{L_\infty}\right]\left(\Embb\left[\Lcal(x_k)\right] - \Lcal^*\right) \label{eq:app_signsgd_proof_result_3}
    \end{align}
    and $\{x_k\}$ converges to a unique optimal point in expectation at a linear rate.
\end{enumerate}
\end{theorem}
\begin{proof}
    Under the stated smoothness assumption, the change in the loss function in iteration $k$ can be bounded using Proposition 2 from \citet{balles2020geometry} as
    \begin{align*}
        \Lcal(x_{k+1}) - \Lcal(x_k)
        &\leq - \eta_k \nabla \Lcal(x_k)^T \sign(g_k) + \frac{L_{\infty}}{2}\|\eta_k\sign(g_k)\|_{\infty}^2 \leq - \eta_k \nabla \Lcal(x_k)^T \sign(g_k) + \eta_k^2 \frac{L_{\infty}}{2}.
    \end{align*}
    where the last inequality follows from $\|\sign(g_k)\|_{\infty} \leq 1$. Taking expectation of the above bound given $x_k$ and using \cref{lem:app_general_norm_inner_product} and \cref{lem:app_signSGD_error_bound},

    \begin{align}
        \Embb_k\left[\Lcal(x_{k+1}) \right] - \Lcal(x_k)
        &\leq - \eta_k\left\|\nabla\Lcal(x_{k})\right\|_1 + \frac{\eta_k\|\sigma_k\|_1}{\sqrt{B_k}} + \eta_k^2 \frac{L_{\infty}}{2} 
        = - \eta_k(1 - \theta)\left\|\nabla\Lcal(x_{k})\right\|_1  + \eta_k^2 \frac{L_{\infty}}{2},  \label{eq:signSGD_pre_telescope}
    \end{align}
    where the equality follows from the defined batch size.

    Taking the total expectation of \Cref{eq:signSGD_pre_telescope} with respect to the entire trajectory of updates, we obtain a telescopic sum for $k = 0, 1, 2, \dots, K-1$, 
    \begin{align*}
        \Lcal(x_0) - \Lcal^*
        &\geq \Lcal(x_0) - \Embb\left[\Lcal(x_K)\right] \nonumber \\
        &= \sum_{k=0}^{K-1} \Embb\left[\Lcal(x_{k}) - \Lcal(x_{k+1})\right] \nonumber \\
        &\geq (1 - \theta) \left[ \sum_{k=0}^{K-1} \eta_k \left\|\nabla\Lcal(x_{k})\right\|_1 - \frac{L_\infty}{2}\sum_{k=0}^{K-1} \eta_k^2\right],
    \end{align*}
    where rearranging the in above using $1 - \theta > 0$ yields,
    \begin{align}
        \Embb\left[ \sum_{k=0}^{K-1} \eta_k \left\|\nabla\Lcal(x_{k})\right\|_1\right]
        \leq
        \frac{1}{(1-\theta)} \left[\Lcal(x_0) - \Lcal^* + \frac{L_\infty}{2}\sum_{k=0}^{K-1} \eta_k^2\right].
        \label{eq:app_signsgd_telescope}
    \end{align}
    Substituting the learning rate as $\eta_k = (1 - \theta) \frac{\|\nabla \Lcal_\infty\|_1}{L_\infty}$
    Substituting the learning rate as $\eta_k = \eta = \frac{1}{\sqrt{L_{\infty} K}}$ in \Cref{eq:app_signsgd_telescope} and multiplying by $\frac{1}{\sqrt{K}}$ yields \Cref{eq:app_signsgd_proof_result_1}.

    If the learning rate satisfies $\sum_{k=0}^\infty \eta_k = \infty$ and $\sum_{k=0}^\infty \eta_k^2 < \infty$, we can infer that $\lim_{K\rightarrow \infty} \sum_{k=0}^{K-1} \eta_k  \Embb\left[ \left\|\nabla\Lcal(x_{k})\right\|_1\right] < \infty$ and \Cref{eq:app_signsgd_proof_result_2} follows.

    If the learning rate is chosen as $\eta_k = \frac{(1 - \theta)\|\nabla \Lcal(x_k)\|_1}{L_\infty}$ (now dependent on the trajectory of updates), substituting in \cref{eq:signSGD_pre_telescope} yields,
    \begin{align}
        \Embb_k\left[\Lcal(x_{k+1}) \right] - \Lcal(x_k)
        \leq -\frac{(1 - \theta)^2\|\nabla \Lcal(x_k)\|_1^2}{2L_\infty}, \label{eq:signnsgd_gns_step_size_ineq_intermediate}
    \end{align}
    where taking the total expectation and performing a telescopic sum as done for \cref{eq:app_signsgd_telescope} results in \cref{eq:app_signsgd_proof_result_gns_lr}.
    
    If $\Lcal(x)$ is also strongly convex, from \cref{eq:signnsgd_gns_step_size_ineq_intermediate} as $1 - \theta > 0$,
    \begin{align*}
        \Embb_k\left[\Lcal(x_{k+1}) \right] - \Lcal(x_k)
        \leq -\frac{(1 - \theta)^2\|\nabla \Lcal(x_k)\|_1^2}{2L_\infty}
        \leq -\frac{(1 - \theta)^2\mu_\infty}{L_\infty} \left(\Lcal(x_k) - \Lcal^*\right).
    \end{align*}
    Taking total expectation of the above and rearranging the inequality yields \Cref{eq:app_signsgd_proof_result_3}.
\end{proof}

\begin{theorem} \label{th:convergence_specGD_CBS}
Assume the objective function $\Lcal(X) = \Embb \left[\lcal(X; \xi)\right] : \Rmbb^{m \times n} \rightarrow \Rmbb$ is smooth with respect to the spectral norm ($\|\cdot\|_{\Scal_\infty}$), i.e.,
\begin{equation*}
    \left\|\nabla \Lcal(X) - \nabla \Lcal(X')\right\|_{\Scal_\infty} \leq L_{\Scal_\infty} \|X - X'\|_{\Scal_1} \quad \forall X, X' \in \Rmbb^{m \times n},
\end{equation*}
where $\|\cdot\|_{\Scal_1}$ is the nuclear norm, with $L_{\Scal_\infty} > 0$ and bounded below by the optimal value $\Lcal^*$. Let $\{X_k\}$ be the sequence generated by the update rule $X_{k+1} = X_k - \eta_k \matsign(G_k)$, where $G_k = \tfrac{1}{B_k}\sum_{\xi \in \Scal_k} \nabla \lcal(X_k; \xi)$ is a gradient estimate computed over a mini-batch of $B_k = \frac{\|C_{\text{row}, k}^{1/2}\|_{\Scal_1}^2}{\theta^2\|\nabla\Lcal(X_{k})\|_{\Scal_1}^2}$ i.i.d samples and $C_{\text{row}, k} = \Embb_k\left[ \left(\nabla \Lcal(X_k) - \nabla \lcal(X_k; \xi)\right)\left(\nabla \Lcal(X_k) - \nabla \lcal(X_k; \xi)\right)^\top\right]$ is the row-wise aggregated covariance matrix.

\begin{enumerate}
    \item If the learning rate is chosen as $\eta_k = \eta = \frac{1}{\sqrt{L_{\Scal_\infty} K}}$ for $K \geq 1$,
    \begin{align}
        \Embb\left[ \frac{1}{K}\sum_{k=0}^{K-1} \left\|\nabla\Lcal(x_{k})\right\|_{\Scal_1}\right]\leq \frac{\sqrt{L_{\Scal_\infty}}}{(1 -  \theta)\sqrt{K}}\left[\Lcal(x_0) - \Lcal^* + \frac{1}{2}\right].
        \label{eq:app_specsgd_proof_result_1}
    \end{align}

    \item If the learning rate is chosen such that  $\sum_{k=0}^\infty \eta_k = \infty$ and $\sum_{k=0}^\infty \eta_k^2 < \infty$,
    \begin{align}
        \liminf_{k \rightarrow \infty} ~ \Embb[\|\nabla \Lcal(X_k)\|_{\Scal_\infty}] = 0. \label{eq:app_specsgd_proof_result_2}
    \end{align}
    
    \item If the learning rate is chosen as $\eta_k = \frac{(1 - \theta)\|\nabla \Lcal(X_k)\|_{\Scal_1}}{L_{\Scal_\infty}}$, then for $K \geq 1$,
    \begin{align}
        \Embb\left[ \frac{1}{K}\sum_{k=0}^{K-1} \left\|\nabla\Lcal(X_{k})\right\|_{\Scal_1}^2\right]\leq \frac{2L_{\Scal_\infty}}{(1 -  \theta)^2K}(\Lcal(X_0) - \Lcal^*). \label{eq:app_specsgd_proof_result_gns_lr}
    \end{align}
    
    If the objective function is also $\mu_{\Scal_\infty}-$strongly convex, i.e.,
    \begin{align*}
        2 \mu_{\Scal_\infty} \left(\Lcal (X)  - \Lcal^*\right)\leq \left\|\nabla \Lcal (X)\right\|_{\Scal_1}^2 \quad \forall X \in \Rmbb^{m \times n},
    \end{align*}
    with $\mu_{\Scal_\infty} > 0$, then
    \begin{align}
        \Embb\left[\Lcal(X_{k+1})\right] - \Lcal^*
        \leq  
        \left[1 - \left(1 - \theta\right)^2\frac{\mu_{\Scal_\infty}}{L_{\Scal_\infty}}\right]\left(\Embb\left[\Lcal(X_k)\right] - \Lcal^*\right) \label{eq:app_specsgd_proof_result_3}
    \end{align}
    and $\{X_k\}$ converges to a unique optimal point in expectation at a linear rate.
\end{enumerate}
\end{theorem}
\begin{proof}
    Under the stated smoothness assumption, from \citet{carlson2015preconditioned}, the change in the loss function in iteration $k$ can be bounded as
    \begin{align}
        \Lcal(X_{k+1}) - \Lcal(X_k) 
        &\leq -\eta_k \langle \nabla\Lcal(X_{k}), \matsign(G_k)\rangle + \frac{L_{\Scal_\infty}}{2} \|\eta_k \matsign(G_k)\|_{\Scal_\infty}^2 
        \nonumber \\
        &= -\eta_k \langle \nabla\Lcal(X_{k}), \matsign(G_k)\rangle + \eta_k^2\frac{L_{\Scal_\infty}}{2}, \nonumber
    \end{align}
    where the equality follows from $\|\matsign(G_k)\|_{\Scal_\infty} = 1$. 
    Taking expectation of the above bound given $x_k$ and using \cref{lem:app_general_norm_inner_product} and \cref{lem:app_specGD_error_bound},

    \begin{align}
        \Embb_k\left[\Lcal(X_{k+1})\right] - \Lcal(X_k)
        \leq -\eta_k\|\nabla\Lcal(X_{k})\|_{\Scal_1} + \eta_k \frac{\|C_{\text{row}, k}^{1/2}\|_{\Scal_1}}{\sqrt{B_k}} + \eta_k^2 \frac{L_{\Scal_\infty}}{2}
        = -\eta_k (1 - \theta) \left\|\nabla\Lcal(x_{k})\right\|_{\Scal_1} + \eta_k^2 \frac{L_{\Scal_\infty}}{2},  \label{eq:SpecGD_pre_telescope}
    \end{align}
    where the equality follows from the defined batch size. 
    
    Following the same procedure as \cref{th:convergence_sign_CBS} of taking the total expectation of \Cref{eq:SpecGD_pre_telescope} with respect to the entire trajectory of updates, taking the telescopic sum for $k = 0, 1, 2, \dots, K-1$ and rearranging,
    \begin{align}
        \Embb\left[ \sum_{k=0}^{K-1} \eta_k \left\|\nabla\Lcal(X_{k})\right\|_{\Scal_1}\right]
        \leq
        \frac{1}{(1-\theta)}\left[\Lcal(X_0) - \Lcal^* + \frac{L_{\Scal_{\infty}}}{2}\sum_{k=0}^{K-1} \eta_k^2\right].
        \label{eq:app_specsgd_telescope}
    \end{align}
    
    Substituting the learning rate as $\eta_k = \eta = \frac{1}{\sqrt{L_{\Scal_\infty} K}}$ in \Cref{eq:app_specsgd_telescope} and multiplying by $\frac{1}{\sqrt{K}}$ yields \Cref{eq:app_specsgd_proof_result_1}.

    If the learning rate satisfies $\sum_{k=0}^\infty \eta_k = \infty$ and $\sum_{k=0}^\infty \eta_k^2 < \infty$, we can infer that $\lim_{K\rightarrow \infty} \sum_{k=0}^{K-1} \eta_k  \Embb\left[ \left\|\nabla\Lcal(X_{k})\right\|_{\Scal_1}\right] < \infty$ and \Cref{eq:app_specsgd_proof_result_2} follows.

    
     If the learning rate is chosen as $\eta_k = \frac{(1 - \theta)\|\nabla \Lcal(X_k)\|_{\Scal_1}}{L_{\Scal_\infty}}$ (now dependent on the trajectory of updates), substituting in \cref{eq:SpecGD_pre_telescope} yields,
    \begin{align}
        \Embb_k\left[\Lcal(X_{k+1}) \right] - \Lcal(X_k)
        \leq -\frac{(1 - \theta)^2\|\nabla \Lcal(X_k)\|_{\Scal_1}^2}{2L_{\Scal_\infty}},
        \label{eq:specsgd_gns_step_size_ineq_intermediate}
    \end{align}
    where taking the total expectation and performing a telescopic sum as done for \cref{eq:app_specsgd_telescope} results in \cref{eq:app_specsgd_proof_result_gns_lr}.
    
    If $\Lcal(X)$ is also strongly convex, from \cref{eq:specsgd_gns_step_size_ineq_intermediate} as $1 - \theta > 0$,
   \begin{align*}
        \Embb_k\left[\Lcal(X_{k+1}) \right] - \Lcal(X_k)
        \leq -\frac{(1 - \theta)^2\|\nabla \Lcal(X_k)\|_{\Scal_1}^2}{2L_{\Scal_\infty}}
        \leq -\frac{(1 - \theta)^2\mu_{\Scal_\infty}}{L_{\Scal_\infty}} \left(\Lcal(X_k) - \Lcal^*\right)
    \end{align*}
    Taking total expectation of the above and rearranging the inequality yields \Cref{eq:app_specsgd_proof_result_3}.

\end{proof}
\section{Alternative Definitions of Critical Batch Size}
\label{appendix:alt-cbs}

This appendix formalizes several common ``turning-point'' definitions of CBS. For \signsgd and \specgd, these definitions differ only by constant factors while using the same intrinsic scale of GNS.

Recall from \Cref{subsec:cbs_derivations}, \Cref{eq:delta_star_summary}, for the non-Euclidean optimizers analyzed, the optimal single-step expected loss improvement takes the form:
\begin{equation}
\label{eq:appendix_generic_delta}
    \Delta^*(B) = \Delta^*_{\max} \left(1 - \sqrt{\frac{\Bcal}{B}}\right)^2, \quad \text{for } B > \Bcal,
\end{equation}
where $\Delta^*_{\max} = \lim_{B \to \infty} \Delta^*(B)$ is the deterministic limit and $\Bcal$ is the geometry-dependent GNS:
\begin{align}
\label{eq:appendix_bcal_instances}
    \Bcal_{\ell_1} &= \left(\frac{\|\sigma\|_1}{\|g\|_1}\right)^2
    && \text{(\signsgd, $\ell_\infty$ geometry)}, \\
    \Bcal_{\Scal_1} &= \left(\frac{\|C_{\text{row}, k}^{1/2}\|_{\Scal_1}}{\|G\|_{\Scal_1}}\right)^2
    && \text{(\specgd,  Schatten-$\infty$ norm)},
\end{align}
for our two non-Euclidean cases. \Cref{eq:appendix_generic_delta} arises by factoring the constants in the
optimized one-step expressions (e.g., $\|\nabla \Lcal(x_k)\|_1$) that do not
affect the location of turning points in $B$.

\begin{definition}[Fraction-of-Maximum CBS]
    The fraction-of-maximum CBS is defined as the batch size $B_{\CBS}(\kappa)$ required to achieve a fraction $\kappa \in (0,1)$ of the maximum possible improvement:
    \begin{equation}
    \label{eq:appendix_cbs_kappa_def}
        B_{\CBS}(\kappa) := \inf \left\{ B : \Delta^*(B) \geq \kappa \Delta^*_{\max} \right\}.
    \end{equation}
\end{definition}

\noindent \textbf{Derivation.}
Solving $\left(1 - \sqrt{\Bcal/B}\right)^2 = \kappa$ for $B$ yields:
\begin{equation}
    \label{eq:appendix_cbs_kappa_solution}
    B_{\CBS}(\kappa) = \left(\frac{1}{1 - \sqrt{\kappa}}\right)^2 \Bcal.
\end{equation}
This confirms that the CBS is a constant scaling of the GNS.


\paragraph{Connection to the $\theta$-parameterization.}
If we select $B=\theta^{-2}\Bcal$, as specified in \Cref{eq:gns_to_batch_size}, then 
\[
\frac{\Delta^*(B)}{\Delta^*(\infty)} = (1-\theta)^2.
\]

Thus, choosing $\theta$ is equivalent to choosing the fraction
$\kappa=(1-\theta)^2$ in \Cref{eq:appendix_cbs_kappa_def,eq:appendix_cbs_kappa_solution}, recovering $B=\theta^{-2}\Bcal$.

\begin{definition}[Inflection Point CBS]
    The inflection point CBS is defined as the batch size $B_{\mathrm{infl}}$ corresponding to the inflection point of the improvement curve, where the rate of marginal gain begins to decrease:
    \begin{equation}
        B_{\mathrm{infl}} := \left\{ B : \frac{d^2}{dB^2} \Delta^*(B) = 0 \right\}.
    \end{equation}
\end{definition}

\noindent \textbf{Derivation.}
Let $f(B) = (1 - \sqrt{\Bcal/B})^2$. Differentiating twice with respect to $B$ yields:
\begin{align}
    f'(B) &= \Bigl(1 - \sqrt{\tfrac{\Bcal}{B}}\Bigr)\cdot
            \Bigl(\sqrt{\Bcal}\,B^{-3/2}\Bigr), \\
    f''(B) &= -\frac{3}{2}\sqrt{\Bcal}\,B^{-5/2} + 2\Bcal\,B^{-3}.
\end{align}
Setting $f''(B)=0$ and solving for $B$ gives
\begin{equation}
\label{eq:appendix_bcurv}
    B_{\mathrm{curv}}
    = \frac{16}{9}\,\Bcal.
\end{equation}

\begin{definition}[Maximum Efficiency CBS]
    The maximum efficiency CBS is defined as the batch size $B_{\mathrm{eff}}$ that maximizes the improvement per sample (computational efficiency):
    \begin{equation}
        B_{\mathrm{eff}} := \arg\max_{B \geq 1} \frac{\Delta^*(B)}{B}.
    \end{equation}
\end{definition}

\noindent \textbf{Derivation.}
Let $h(B) = \frac{1}{B} (1 - \sqrt{\Bcal/B})^2$ and $t = \sqrt{\Bcal/B}$. Then we maximize $t^2 (1 - t)^2$ for $t \in (0, 1)$. The stationarity condition yields $t = 1/2$. Substituting back to $B$:
\begin{equation}
\label{eq:appendix_beff}
    B_{\mathrm{eff}} = 4 \Bcal.
\end{equation}

\paragraph{Instantiations for \signsgd and \specgd.}

Plugging \Cref{eq:appendix_bcal_instances} into \Cref{eq:appendix_bcurv,eq:appendix_beff} yields:
\begin{align*}
    B^{\mathrm{sign}}_{\mathrm{inf}} &= \frac{16}{9}\left(\frac{\|\sigma\|_1}{\|g\|_1}\right)^2,
    &
    B^{\mathrm{sign}}_{\mathrm{eff}} &= 4\left(\frac{\|\sigma\|_1}{\|g\|_1}\right)^2, \\
    B^{\mathrm{spec}}_{\mathrm{inf}} &= \frac{16}{9}\left(\frac{\|C_{\text{row}, k}^{1/2}\|_{\Scal_1}}{\|G\|_{\Scal_1}}\right)^2,
    &
    B^{\mathrm{spec}}_{\mathrm{eff}} &= 4\left(\frac{\|C_{\text{row}, k}^{1/2}\|_{\Scal_1}}{\|G\|_{\Scal_1}}\right)^2.
\end{align*}

\paragraph{Comparison with Euclidean SGD.}
From \Cref{eq:delta_star_summary} and \citep{mccandlish2018empirical}, \sgd follows a saturation curve of the form $\Delta^*(B) \propto (1 + \Bcal_{\ell_2}/B)^{-1}$.
This function is strictly concave for $B > 0$, meaning it lacks an inflection point. Furthermore, the efficiency ratio $\Delta^*(B)/B$ is strictly decreasing, meaning the maximum efficiency occurs at $B=1$.
Thus, the definitions for $B_{\mathrm{infl}}$ and $B_{\mathrm{eff}}$ are ill-defined for standard SGD. This necessitates the use of the Fraction-of-Maximum definition (commonly with $\kappa=0.5$, yielding $B=\Bcal_{\ell_2}$) in the Euclidean setting.
In contrast, the non-Euclidean geometries of \signsgd and \specgd induce an initial convex phase in optimization, making $B_{\mathrm{infl}}$ and $B_{\mathrm{eff}}$ valid alternate targets.
All three definitions, however, result in $B \propto \Bcal$. In our experiments, we found the tuned value of $\theta$ to align closely with the maximum efficiency CBS for language models.

\section{Experimental Setup} \label{app:exp_setup}


\subsection{Language Models} \label{app:exp_setup_language}
For our language model experimental runs, we utilize a learning rate schedule with a linear warm-up for the first 15\% of the total sample budget, followed by a cosine decay schedule. For the adaptive approach, we maintain a constant initial batch period ($I$) that matches the learning rate warm-up duration, measure the GNS every 100 iterations ($F=100$), and fix the scaling coefficient at $\theta = 0.6$ for all configurations. \cref{tab:hyperparameters_llama} summarizes the search space for hyperparameter tuning.

\begin{table}[ht]
\caption{Summary of hyperparameter search space for language modeling experiments.
}
\label{tab:hyperparameters_llama}
\begin{center}
\begin{small}
\begin{sc}
\begin{tabular}{llc}
\toprule
Method & Hyperparameter & Parameter Set  \\
\midrule
\multirow{3}{*}{\cref{alg:structure}} & $(\beta^N$, $\beta^M)$ & $(0.9, 0.9)$ \\
& $\theta$  & 0.6 \\
& $F$  & 100 iterations \\
\midrule
\multirow{6}{*}{All Methods} & Learning rate schedule & cosine \\
& Warmup fraction \% & 15\% \\
& Learning rate min fraction \% & 0.0 \\
& Decoupled weight decay & 0.1 \\
& \multirow{2}{*}{Learning rate} &  \multirow{2}{*}{\begin{tabular}[c]{@{}c@{}}$\{1e^{-4}, 2e^{-4}, 4e^{-4}, 6e^{-4}, 8e^{-4},$ \\  $\quad 1e^{-3}, 2e^{-3}, 4e^{-3}, 8e^{-3}, 1e^{-2}$\}\end{tabular}} \\
& & \\
\midrule
\multirow{2}{*}{\signum} & \multirow{2}{*}{EMA parameter $\beta$}  & \multirow{2}{*}{\begin{tabular}[c]{@{}c@{}}$\{0.7, 0.8, 0.85, 0.875, 0.9, 0.925$ \\  $\quad 0.95, 0.975, 0.9875, 0.999\}$\end{tabular}} \\ \\
\midrule
\multirow{5}{*}{\adamw} & \multirow{2}{*}{$\beta_1$}& \multirow{2}{*}{\begin{tabular}[c]{@{}c@{}}$\{0.7, 0.8, 0.85, 0.875, 0.9, 0.925$ \\  $\quad 0.95, 0.975, 0.9875, 0.999\}$\end{tabular}} \\ \\
& \multirow{2}{*}{$\beta_2$}& \multirow{2}{*}{\begin{tabular}[c]{@{}c@{}}$\{0.7, 0.8, 0.85, 0.875, 0.9, 0.925$ \\  $\quad 0.95, 0.975, 0.9875, 0.999\}$\end{tabular}} \\ \\
& $\epsilon$ & $1e^{-8}$ \\
\midrule
\multirow{3}{*}{\muon} & Newton Schulz Iterations & $5$ \\
& \multirow{2}{*}{Momentum $\beta$}  & \multirow{2}{*}{\begin{tabular}[c]{@{}c@{}}$\{0.7, 0.8, 0.85, 0.875, 0.9, 0.925$ \\  $\quad 0.95, 0.975, 0.9875, 0.999\}$\end{tabular}} \\ \\
\bottomrule
\end{tabular}
\end{sc}
\end{small}
\end{center}
\end{table}

\subsection{Vision Models} \label{app:exp_setup_vision}

We run two image classification workloads that cover common vision architectures. First, we train SimpleViT \citep{beyer2022better} over the Imagewoof dataset \citep{Howard_Imagewoof_2019} which isolates the transformer setting in a vision task. Second, we train ResNet-18 over the CIFAR-10 dataset to cover CNNs with residual connections. 

Across both workloads, we compare constant batch training to the adaptive strategy in Algorithm~\ref{alg:structure}. For ResNet-18 over the CIFAR-10, we train for 100 epochs, and for SimpleViT over the Imagewoof we use a constant batch baseline of $B=128$. For adaptive batching, we fix $(\beta^N,\beta^M)=(0.9,0.9)$ and sweep the tolerance parameter $\theta$. The batch size is updated periodically with period $F$ set to 1 epoch for ResNet-18 over the CIFAR-10 and 3 epochs for SimpleViT over the Imagewoof.
\cref{tab:hyperparameters_cifar10,tab:hyperparameters_imagewoof} summarize the hyperparameter search spaces.


\begin{table}[ht]

\caption{Summary of hyperparameter search space used for training SimpleViT over Imagewoof dataset.}
\label{tab:hyperparameters_imagewoof}
\begin{center}
\begin{small}
\begin{sc}
\begin{tabular}{llc}
\toprule
Method & Hyperparameter & Parameter Set  \\
\midrule
\multirow{3}{*}{\cref{alg:structure}}
& $(\beta^N$, $\beta^M)$ & $(0.9, 0.9)$ \\
& $\theta$  & \{0.25, 0.5, 1, 2\} \\
& $F$  & 3 epochs \\
\midrule
\multirow{6}{*}{All Methods} & Epochs & $100$ \\
& Momentum & $\{0.925, 0.95, 0.975, 0.9875, 0.999\}$ \\
& Learning rate schedule & cosine \\
& Warmup fraction \% & $4\%$ \\
& Learning rate min value & $1e^{-6}$\\ 
& Decoupled weight decay & $0.0001$ \\
\midrule
\multirow{2}{*}{MSGD / SGD} & Learning Rate & $\{0.005, 0.007, 0.01, 0.023, 0.05, 0.07, 0.1, 0.2\}$ \\
& Momentum & ${0.925, 0.95, 0.975, 0.9875, 0.999}$\\
\midrule
\multirow{2}{*}{\signsgd} & Learning Rate & $\{0.001, 0.005, 0.007, 0.01, 0.05\}$ \\
& Epochs & $200$ \\
\midrule
\multirow{2}{*}{\signum} & Learning Rate & $\{1e^{-4}, 5e^{-4}, 1e^{-3}, 5e^{-3}, 1e^{-2}\}$ \\
& EMA Parameter $\beta$ & $\{0.7, 0.8, 0.85, 0.875, 0.9, 0.925, 0.95, 0.975, 0.9875, 0.999\}$ \\
\midrule
\multirow{4}{*}{\adamw} & Learning Rate & $\{1e^{-4}, 3e^{-4}, 6e^{-4}, 1e^{-3}, 3e^{-3}, 6e^{-3}, 1e^{-2}\}$ \\
& $\beta_1$ & $\{0.925, 0.95, 0.975, 0.9875, 0.999\}$ \\
& $\beta_2$ & $0.999$ \\
& $ \epsilon$ & $1e^{-8}$ \\
\midrule
\multirow{2}{*}{\muon} & Learning Rate & $\{6e^{-4}, 1e^{-3}, 3e^{-3}, 6e^{-3}, 1e^{-2}\}$ \\
& {Momentum $\beta$} & $0.95$\\
\bottomrule
\end{tabular}
\end{sc}
\end{small}
\end{center}
\end{table}

\begin{table}[h]
\caption{Summary of hyperparameter search spaces for training ResNet-18 over CIFAR-10 dataset.}
\label{tab:hyperparameters_cifar10}
\begin{center}
\begin{small}
\begin{sc}
\begin{tabular}{llc}
\toprule
Method & Hyperparameter & Parameter Set  \\
\midrule
\multirow{3}{*}{\cref{alg:structure}} 
& $(\beta^N$, $\beta^M)$ & $(0.9, 0.9)$ \\
& $\theta$  & \{0.125, 0.25, 0.5, 1, 2.0, 4.0\} \\
& $F$  & 1 epoch \\
& Initial Batch Size & $\{64, 256\}$ \\
\midrule
\multirow{6}{*}{All Methods} & Epochs & $100$ \\
& Learning rate schedule & cosine \\
& Warmup fraction \% & $5\%$ \\
& Learning rate min value & $1e^{-6}$ \\
& Decoupled weight decay & $0.0$\\
\midrule
\multirow{2}{*}{MSGD/SGD} & Learning Rate & $\{0.1, 0.2, 0.3, 0.5, 0.75, 1.0\}$ \\
& Momentum & $\{0.9, 0.925, 0.95, 0.975\}$\\
\midrule
\multirow{1}{*}{\signsgd} & Learning Rate & $\{1e^{-4}, 1e^{-3}, 1e^{-2}\}$ \\
\midrule
\multirow{2}{*}{\signum} & Learning Rate & $\{1e^{-4}, 1e^{-3}, 1e^{-2}, 1e^{-1}\}$\\
& EMA Parameter $\beta$ & $\{0.9, 0.925, 0.95, 0.95\}$\\
\midrule
\multirow{4}{*}{\adamw} & Learning Rate & $\{1e^{-4}, 3e^{-4}, 1e^{-3}, 3e^{-3}, 1e^{-2}, 3e^{-2}, 1e^{-1}\}$\\
& $\beta_1$ & $0.9$ \\
& $\beta_2$ & $0.999$ \\
& $\epsilon$ & $1e^{-8}$ \\
\midrule
\multirow{3}{*}{\muon} & Learning Rate & $\{1e^{-4}, 1e^{-3}, 1e^{-2}, 1e^{-1}\}$ \\
& {Momentum $\beta$} & $\{0.9, 0.925, 0.95, 0.95\}$\\
\bottomrule
\end{tabular}
\end{sc}
\end{small}
\end{center}
\end{table}

\section{Additional Experimental Results}
In this section, we provide additional numerical experiments to support the empirical claims made in the main body.

\subsection{Language Models} \label{app:languange}

\begin{table}[ht]
\renewcommand{\arraystretch}{1.3}
\caption{Comparison of constant versus adaptive batch sizes for a 160M Llama 3 trained for 3.2B tokens (10 seeds) over the C4 dataset using \signsgd. Results show that adaptive batch sizes achieve improved training efficiency over constant-batch baselines and better validation loss with learning rate scaling.}
\label{tab:signsgd_hueristics_effect}
\begin{center}
\begin{small}
\begin{sc}
\begin{tabular}{lccccc}
\toprule
\multirow{2}{*}{\begin{tabular}[c]{@{}c@{}}Batch Size \\ Method\end{tabular}} &
\multirow{2}{*}{\begin{tabular}[c]{@{}c@{}}Initial \\ Batch Size \end{tabular}} &
\multirow{2}{*}{\begin{tabular}[c]{@{}c@{}}Learning \\ Rate Scaling\end{tabular}} &
\multirow{2}{*}{\begin{tabular}[c]{@{}c@{}}Validation \\ Loss\end{tabular}} &
\multirow{2}{*}{\begin{tabular}[c]{@{}c@{}}Median \\ Training Steps\end{tabular}} \\ \\
\midrule
\multirow{3}{*}{Constant} & 64 & - & $4.1390 \pm 0.0235$& 24415\\ \cmidrule{2-5}
& 128 & - & $3.9963 \pm 0.0098$ & 8426 \\ \cmidrule{2-5}
& 256 & - & $ 3.9357 \pm 0.0057$& 6103\\
\midrule
\multirow{6}{*}{Adaptive} & \multirow{2}{*}{64} & False & $3.6857 \pm 0.0085$ & 9285 \\ \cmidrule{3-5}
& & True & $\textbf{3.6817} \pm 0.0094$ & 9950 \\ \cmidrule{2-5}
& \multirow{2}{*}{128} & False & $3.7082 \pm 0.0068$ & 7617\\ \cmidrule{3-5}
& & True & $\textbf{3.7029} \pm 0.0094$ & 7826\\ \cmidrule{2-5}
& \multirow{2}{*}{256} & False & $3.8539 \pm 0.0039$ & 5007 \\ \cmidrule{3-5}
& & True & $\textbf{3.8396} \pm 0.0012$ & 5274\\
\bottomrule
\end{tabular}
\end{sc}
\end{small}
\end{center}
\end{table}

\begin{table}[ht]
\renewcommand{\arraystretch}{1.3}
\caption{Comparison of adaptive batch sizes using $\lcal_1$ and $\lcal_2$ GNS with \signsgd for 160M Llama 3 trained for 3.2B tokens (10 seeds) over the C4 dataset. Results show that adaptive batch sizes using $\lcal_1$ GNS, which aligns with the optimizer geometry, performs better.}
\label{tab:GNS_l2_l1_comparison_llama_signsgd}
\begin{center}
\begin{small}
\begin{sc}
\begin{tabular}{lcccc}
\toprule
\multirow{2}{*}{\begin{tabular}[c]{@{}c@{}}GNS\end{tabular}} &
\multirow{2}{*}{\begin{tabular}[c]{@{}c@{}}Initial \\ Batch Size \end{tabular}} &
\multirow{2}{*}{\begin{tabular}[c]{@{}c@{}}Validation \\ Loss\end{tabular}} &
\multirow{2}{*}{\begin{tabular}[c]{@{}c@{}}Median \\ Training Steps\end{tabular}} \\ \\
\midrule
\multirow{3}{*}{$\Bcal_{\lcal_2}$} & 64 & $3.7431 \pm 0.0402$& 10798\\ \cmidrule{2-4}
& 128  & $3.7752 \pm 0.0088$ &  8869\\ \cmidrule{2-4}
& 256 & $3.8906 \pm 0.0117$& 5574\\
\midrule
\multirow{3}{*}{$\Bcal_{\lcal_1}$} & 64 & $\textbf{3.6817} \pm 0.0094$ & 9950 \\ \cmidrule{2-4}
& 128 & $\textbf{3.7029} \pm 0.0094$ & 7826\\ \cmidrule{2-4}
& 256 & $\textbf{3.8396} \pm 0.0012$ & 5274\\
\bottomrule
\end{tabular}
\end{sc}
\end{small}
\end{center}
\end{table}

This section presents extended results for the 160M Llama 3 model trained using \signsgd.
\cref{tab:signsgd_hueristics_effect} evaluates our adaptive approach against baseline strategies utilizing constant batch sizes.
The adaptive batch size strategy consistently outperforms constant baselines, reaching lower validation loss in significantly fewer training steps.
The results also indicate a clear performance gain when scaling the learning rate in tandem with the batch size.

\cref{tab:GNS_l2_l1_comparison_llama_signsgd} compares the performance of the adaptive strategy when using the $\lcal_1$ and $\lcal_2$ GNS metrics. For $\lcal_2$, $\theta = 0.3$ is used after tuning.
The results demonstrate the adaptive strategy via $\lcal_1$ GNS yields better results.

To further assess the robustness of our approach, \cref{fig:signsgd_diagnosis_plots} visualizes training trajectories across multiple seeds for the adaptive strategy (starting at a batch size of 64) and \cref{tab:GNS_l1_theta_sensitivity_llama_signsgd} provides sensitivity results for $\theta$.
While individual seeds produce unique batch size paths in \cref{fig:signsgd_diagnosis_plots}, based on their specific gradient noise profiles, all runs yield better performance than the constant batch size baselines.
This consistency demonstrates that the proposed method effectively adapts the training schedule to the observed data distribution. In \cref{tab:GNS_l1_theta_sensitivity_llama_signsgd}, one can see that even with the change in results over different values for $\theta$, the result loss form adaptive strategy via $\lcal_1$ GNS is still better than the best result via $\lcal_2$ GNS.

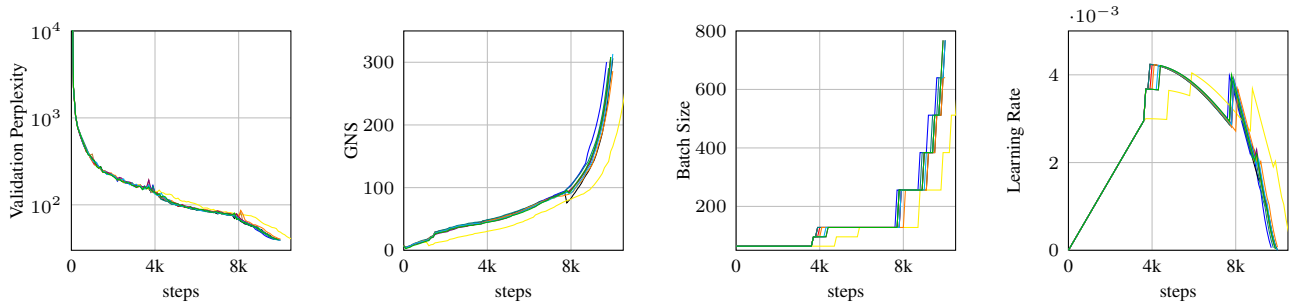
\begin{figure}[ht]
    \centering
    \begin{tikzpicture}
        \begin{groupplot}[
            Llama3_dissect_plots,  
            group style={
                group size=4 by 1,       
                horizontal sep=1.5cm,    
            }
        ]
            \nextgroupplot[
                xlabel={steps},
                ylabel={Validation Perplexity},
                xmin=0, xmax=10500,
                ymin=30, ymax=1e4,
                ymode=log,
                xtick={0, 4000, 8000},
                xticklabels={0, 4k, 8k}, 
            ]
            \addplot table [x=_step, y=perplexity, col sep=comma] {Llama_results/llama_160M_appendix_sign_results/gns_64_dissect_1.csv};
            \addplot table [x=_step, y=perplexity, col sep=comma] {Llama_results/llama_160M_appendix_sign_results/gns_64_dissect_2.csv};
            \addplot table [x=_step, y=perplexity, col sep=comma] {Llama_results/llama_160M_appendix_sign_results/gns_64_dissect_3.csv};
            \addplot table [x=_step, y=perplexity, col sep=comma] {Llama_results/llama_160M_appendix_sign_results/gns_64_dissect_4.csv};
            \addplot table [x=_step, y=perplexity, col sep=comma] {Llama_results/llama_160M_appendix_sign_results/gns_64_dissect_5.csv};
            \addplot table [x=_step, y=perplexity, col sep=comma] {Llama_results/llama_160M_appendix_sign_results/gns_64_dissect_6.csv};
            \addplot table [x=_step, y=perplexity, col sep=comma] {Llama_results/llama_160M_appendix_sign_results/gns_64_dissect_7.csv};
            \addplot table [x=_step, y=perplexity, col sep=comma] {Llama_results/llama_160M_appendix_sign_results/gns_64_dissect_8.csv};
            \addplot table [x=_step, y=perplexity, col sep=comma] {Llama_results/llama_160M_appendix_sign_results/gns_64_dissect_9.csv};
            \addplot table [x=_step, y=perplexity, col sep=comma] {Llama_results/llama_160M_appendix_sign_results/gns_64_dissect_10.csv};

            \nextgroupplot[
                ymode=normal,
                xlabel={steps},
                ylabel={GNS},
                xmin=0, xmax=10500,
                ymin=0, ymax=350,
                xtick={0, 4000, 8000},
                xticklabels={0, 4k, 8k}, 
            ]
            \addplot table [x=_step, y=adaptive_batch_metrics/gns_ema, col sep=comma] {Llama_results/llama_160M_appendix_sign_results/gns_64_dissect_1.csv};
            \addplot table [x=_step, y=adaptive_batch_metrics/gns_ema, col sep=comma] {Llama_results/llama_160M_appendix_sign_results/gns_64_dissect_2.csv};
            \addplot table [x=_step, y=adaptive_batch_metrics/gns_ema, col sep=comma] {Llama_results/llama_160M_appendix_sign_results/gns_64_dissect_3.csv};                \addplot table [x=_step, y=adaptive_batch_metrics/gns_ema, col sep=comma] {Llama_results/llama_160M_appendix_sign_results/gns_64_dissect_4.csv};                \addplot table [x=_step, y=adaptive_batch_metrics/gns_ema, col sep=comma] {Llama_results/llama_160M_appendix_sign_results/gns_64_dissect_5.csv};                
            \addplot table [x=_step, y=adaptive_batch_metrics/gns_ema, col sep=comma] {Llama_results/llama_160M_appendix_sign_results/gns_64_dissect_6.csv};                \addplot table [x=_step, y=adaptive_batch_metrics/gns_ema, col sep=comma] {Llama_results/llama_160M_appendix_sign_results/gns_64_dissect_7.csv};                \addplot table [x=_step, y=adaptive_batch_metrics/gns_ema, col sep=comma] {Llama_results/llama_160M_appendix_sign_results/gns_64_dissect_8.csv};                \addplot table [x=_step, y=adaptive_batch_metrics/gns_ema, col sep=comma] {Llama_results/llama_160M_appendix_sign_results/gns_64_dissect_9.csv};                \addplot table [x=_step, y=adaptive_batch_metrics/gns_ema, col sep=comma] {Llama_results/llama_160M_appendix_sign_results/gns_64_dissect_10.csv};
            
            \nextgroupplot[
                ymode=normal,
                xlabel={steps},
                ylabel={Batch Size},
                xmin=0, xmax=10500,
                ymin=50, ymax=800,
                xtick={0, 4000, 8000},
                xticklabels={0, 4k, 8k}, 
            ]
            \addplot table [x=_step, y=batch_size/global_batch_size, col sep=comma] {Llama_results/llama_160M_appendix_sign_results/gns_64_dissect_1.csv};
            \addplot table [x=_step, y=batch_size/global_batch_size, col sep=comma] {Llama_results/llama_160M_appendix_sign_results/gns_64_dissect_2.csv};
            \addplot table [x=_step, y=batch_size/global_batch_size, col sep=comma] {Llama_results/llama_160M_appendix_sign_results/gns_64_dissect_3.csv};
            \addplot table [x=_step, y=batch_size/global_batch_size, col sep=comma] {Llama_results/llama_160M_appendix_sign_results/gns_64_dissect_4.csv};
            \addplot table [x=_step, y=batch_size/global_batch_size, col sep=comma] {Llama_results/llama_160M_appendix_sign_results/gns_64_dissect_5.csv};
            \addplot table [x=_step, y=batch_size/global_batch_size, col sep=comma] {Llama_results/llama_160M_appendix_sign_results/gns_64_dissect_6.csv};
            \addplot table [x=_step, y=batch_size/global_batch_size, col sep=comma] {Llama_results/llama_160M_appendix_sign_results/gns_64_dissect_7.csv};
            \addplot table [x=_step, y=batch_size/global_batch_size, col sep=comma] {Llama_results/llama_160M_appendix_sign_results/gns_64_dissect_8.csv};
            \addplot table [x=_step, y=batch_size/global_batch_size, col sep=comma] {Llama_results/llama_160M_appendix_sign_results/gns_64_dissect_9.csv};
            \addplot table [x=_step, y=batch_size/global_batch_size, col sep=comma] {Llama_results/llama_160M_appendix_sign_results/gns_64_dissect_10.csv};

            \nextgroupplot[
                ymode=normal,
                xlabel={steps},
                ylabel={Learning Rate},
                xmin=0, xmax=10500,
                ymin=0, ymax=0.005,
                xtick={0, 4000, 8000},
                xticklabels={0, 4k, 8k}, 
            ]
            \addplot table [x=_step, y=lr, col sep=comma] {Llama_results/llama_160M_appendix_sign_results/gns_64_dissect_1.csv};
            \addplot table [x=_step, y=lr, col sep=comma] {Llama_results/llama_160M_appendix_sign_results/gns_64_dissect_2.csv};
            \addplot table [x=_step, y=lr, col sep=comma] {Llama_results/llama_160M_appendix_sign_results/gns_64_dissect_3.csv};                \addplot table [x=_step, y=lr, col sep=comma] {Llama_results/llama_160M_appendix_sign_results/gns_64_dissect_4.csv};                \addplot table [x=_step, y=lr, col sep=comma] {Llama_results/llama_160M_appendix_sign_results/gns_64_dissect_5.csv};                
            \addplot table [x=_step, y=lr, col sep=comma] {Llama_results/llama_160M_appendix_sign_results/gns_64_dissect_6.csv};                \addplot table [x=_step, y=lr, col sep=comma] {Llama_results/llama_160M_appendix_sign_results/gns_64_dissect_7.csv};                \addplot table [x=_step, y=lr, col sep=comma] {Llama_results/llama_160M_appendix_sign_results/gns_64_dissect_8.csv};                \addplot table [x=_step, y=lr, col sep=comma] {Llama_results/llama_160M_appendix_sign_results/gns_64_dissect_9.csv};                \addplot table [x=_step, y=lr, col sep=comma] {Llama_results/llama_160M_appendix_sign_results/gns_64_dissect_10.csv};

        \end{groupplot}
    \end{tikzpicture}
    \vspace{0.1cm} \\
    \caption{Validation perplexity, \emph{exponential moving average} of GNS, batch size and learning rate for 10 seeds for the adaptive strategy starting with an initial batch size of 64 using \signsgd for 160M Llama 3  trained for 3.2B tokens over the C4 dataset. The stability of the validation perplexity across multiple seeds underscores the reliability of our proposed method.}    
    \label{fig:signsgd_diagnosis_plots}
\end{figure}

\begin{table}[ht]
\renewcommand{\arraystretch}{1.3}
\caption{Comparison of adaptive batch sizes using $\lcal_1$ GNS with \signsgd for 160M Llama 3 trained for 3.2B tokens (10 seeds) over the C4 dataset starting at batch size 64 over multiple values of $\theta$. Results show low sensitivity to $\theta$ and the results still performing better than $\lcal_2$ GNS (\cref{tab:GNS_l2_l1_comparison_llama_signsgd}).}
\label{tab:GNS_l1_theta_sensitivity_llama_signsgd}
\begin{center}
\begin{small}
\begin{sc}
\begin{tabular}{ccccc}
\toprule
\multirow{2}{*}{\begin{tabular}[c]{@{}c@{}}Initial \\ Batch Size \end{tabular}} &
\multirow{2}{*}{\begin{tabular}[c]{@{}c@{}}$\theta$\end{tabular}} &
\multirow{2}{*}{\begin{tabular}[c]{@{}c@{}}Validation \\ Loss\end{tabular}} &
\multirow{2}{*}{\begin{tabular}[c]{@{}c@{}}Median \\ Training Steps\end{tabular}} \\ \\
\midrule
\multirow{3}{*}{64} & $0.5$ & $3.6909 \pm 0.0097$& 7650\\ \cmidrule{2-4}
& $0.6$  & $3.6817 \pm 0.0094$ &  9950\\ \cmidrule{2-4}
& $0.7$ & $3.7292 \pm 0.0092$ & 12374\\
\bottomrule
\end{tabular}
\end{sc}
\end{small}
\end{center}
\end{table}


\subsection{Vision Models} \label{app:vision}

This section presents results for the ResNet-18 model trained over the CIFAR-10, comparing the constant batch size baseline at ($B=256$) and ($B=512$) to the proposed adaptive strategy across optimizers, summarized in \cref{tab:cifar10_summary_best_loss}. 
Compared to the baseline ($B=256$), adaptive batching maintains or improves the validation accuracy for five of the six optimizers, with \msgd being the only exception. For the optimizers that reach the baseline threshold, adaptive batching reduces the number of optimization steps by 22.02--72.14\%.
The reduction is consistent for momentum SGD and sign-based methods, and it is largest for \muon, which also attains the best accuracy among the listed optimizers. 

\begin{table}
\begin{center}
\caption{
Validation accuracy and step reduction (\%)  for the ResNet-18 model trained (5
seeds) over the CIFAR-10 dataset. 
The steps reduction (\%) indicates the percentage reduction in steps required to reach the best validation accuracy of the $B = 256$ baseline. The $B = 512$ column is reported for reference. The Adaptive column uses the GNS measured in the dual norm matched to each optimizer's geometry ($\ell_2$ for \sgd/\msgd/, $\ell_1$ for \signsgd/\signum/\adamw, and $\Scal_1$ for \muon).
}
\label{tab:cifar10_summary_best_loss}
\begin{small}
\begin{sc}
\begin{tabular}{lcccc}
\toprule
  \multirow{2}{*}[-0.75ex]{Optimizer} & \multicolumn{3}{c}{validation accuracy} &
\multirow{2}{*}[-0.75ex]{\begin{tabular}[c]{@{}c@{}}Steps \\
Reduction(\%)\end{tabular}} \\ \cmidrule{2-4}
   & B = 256 & B = 512 & Adaptive &  \\
\midrule
\rowcolor{SGDColor}
\msgd & $93.5680 \pm 0.1669$ & $93.4080 \pm 0.2810$ & $93.3060 \pm 0.1270$ & -
\\
\rowcolor{SGDColor}
\sgd & $92.5700 \pm 0.1970$ & $92.0840 \pm 0.1965$ & $93.5600 \pm 0.1378$ & $52.29$ 
\\
\rowcolor{SignSGDColor}
\signsgd & $93.1840 \pm 0.0513$ & $93.2100 \pm 0.2021$ & $93.2320 \pm 0.1979$ & $72.14$
\\
\rowcolor{SignSGDColor}
\signum  & $93.5420 \pm 0.2513$ & $93.7680 \pm 0.0823$ & $93.8160 \pm 0.0940$ & $54.62$
\\
\rowcolor{SignSGDColor}
\adamw & $93.6020 \pm 0.2622$ & $93.9880 \pm 0.1066$ & $93.7580 \pm 0.1612$ & $22.02$
\\
\rowcolor{SpecSGDColor}
\muon & $94.4980 \pm 0.1722$ & $94.6740 \pm 0.2110$ & $94.6420 \pm 0.0750$ & $68.16$ 
\\
\bottomrule
\end{tabular}
\end{sc}
\end{small}
\end{center}
\end{table}

\clearpage
\newpage

\section{Computational Cost of Non-Euclidean GNS Estimation} \label{app:cost_analysis}

This section reports the cost of the $\ell_1$ and $\Scal_1$ GNS estimators. We first present wall-clock measurements for Llama 3 at the 160M and 1B scales, then give a complexity analysis of the additional compute, communication, and memory costs.

\subsection{Wall-Clock Overhead} \label{app:cost_wallclock}

We measured the latency of each training step for the 160M and 1B Llama 3 models on a single $8 \times$H100 (80GB) node using the same DDP configuration as our main experiments.
\cref{tab:cost_wallclock} reports averages over $1{,}000$ optimizer steps with the GNS estimator invoked every $F=100$ steps, as in \cref{tab:hyperparameters_llama}. 
The cost of the additional reductions, the sample variance computation, the matrix square root, and the nuclear-norm evaluation is amortized over $F$ training steps.

\begin{table}[ht]
\caption{Average latency of each training step with and without GNS estimation on a single $8\times$H100 (80GB) node with DDP enabled, averaged over $1{,}000$ optimizer steps with $F=100$.}
\label{tab:cost_wallclock}
\begin{center}
\begin{small}
\begin{sc}
\begin{tabular}{l|c|cc|cc}
\toprule
Configuration & Baseline (s/step) & + $\ell_1$ GNS (s/step) & Overhead & + $\Scal_1$ GNS (s/step) & Overhead \\
\midrule
160M (local BS 16) & $0.447$ & $0.448$ & $+0.2\%$ & $0.509$ & $+13.9\%$ \\
1B (local BS 8)    & $0.622$ & $0.623$ & $+0.3\%$ & $0.869$ & $+39.9\%$ \\
\bottomrule
\end{tabular}
\end{sc}
\end{small}
\end{center}
\end{table}

Computing the $\ell_1$ GNS adds an additional element-wise square operation with one additional \texttt{AllReduce} over the per-element statistics. 
This adds under $1\%$ overhead at both scales.
However, computing the $\Scal_1$ GNS incurs a larger overhead, particularly at the 1B scale, since it requires computing Gram matrices and an eigendecomposition for each parameter tensor.
We leave further memory and performance optimizations of our implementation (FSDP-aware sharding of the Gram matrices, employing efficient iterative matrix-square-root algorithms, e.g., coupled Newton, and GNS computation and communication overlap) to future work.
\subsection{Theoretical Cost Breakdown} \label{app:cost_theoretical}


We summarize the computational, communication, and memory complexities associated with our $\ell_1$ and $\mathcal{S}_1$ GNS estimators in \cref{tab:gns_overhead}. This is based on pseudocode compatible with FSDP provided in \cref{alg:fsdp_l1_gns} and \ref{alg:fsdp_s1_var} for the $\ell_1$ and $\mathcal{S}_1$ GNS estimators, respectively. 
For our analysis, let $P_l = m_l n_l$ (assuming $m_l \le n_l$) denote the parameter count for layer $l$, $P = \sum_l P_l$ denote the total parameter count, and $R$ denote the number of data-parallel ranks.


    

\begin{table*}[t]
\centering
\caption{Summary of additional computational, communication, and peak memory overheads for the proposed non-Euclidean GNS estimators ($\ell_1$ and $\mathcal{S}_1$) under standard Distributed Data Parallel (DDP) and Fully Sharded Data Parallel (FSDP) configurations. Costs are reported per measurement step and all memory buffers are transient. 
}
\label{tab:gns_overhead}
\begin{tabular}{@{} l l p{5.5cm} p{5.5cm} @{}}
\toprule
\textbf{Overhead Type} & \textbf{Distribution} & \textbf{$\ell_1$ GNS Estimator} & \textbf{$\mathcal{S}_1$ GNS Estimator} \\
\midrule
\textbf{Compute} & DDP / FSDP & $\mathcal{O}(P_l)$ element-wise squares & $\mathcal{O}(m_l^2 n_l)$ Gram matrix matmul \\
(Per Layer) & & and variance computation & + $\mathcal{O}(m_l^3)$ eigendecomposition \\
 & & & + $\mathcal{O}(m_l^2 n_l)$ mean-Gram computation \\
\midrule
\textbf{Communication} & DDP & $\mathcal{O}(P)$  \texttt{AllReduce} for & $\mathcal{O}(\sum_l m_l^2)$ \texttt{AllReduce} for\\
(Per Measurement) && squared-gradient statistics & Gram matrix aggregation \\
\addlinespace
& FSDP & $\mathcal{O}(P)$ \texttt{ReduceScatter} for
& $\mathcal{O}(\sum_l m_l^2)$ \texttt{AllReduce} for \\
&& squared-gradient statistics& Gram matrix Aggregation\\
&&+ $\mathcal{O}(1)$ scalar \texttt{AllReduce}& + $\mathcal{O}(P)$ \texttt{AllGather} for \\
&&& full mean gradient reconstruction \\
\midrule
\textbf{Peak Memory} & DDP & $\mathcal{O}(P_l)$ squared-gradient buffer & $\mathcal{O}(m_l^2)$ Gram matrix buffer \\
\addlinespace
(Per Layer) & FSDP & $\mathcal{O}(P_l / R)$ sharded squared & $\mathcal{O}(m_l^2)$ Gram matrix buffer\\
& & mean-gradient buffer & + $\mathcal{O}(P_l)$ reconstructed full\\
& & + $\mathcal{O}(P_l)$ squared-gradient buffer & mean-gradient buffer \\
\bottomrule
\end{tabular}
\end{table*}

\begin{algorithm}[tb]
\caption{$\ell_1$ GNS Estimation with FSDP at step $k$}
\label{alg:fsdp_l1_gns}
\begin{algorithmic}[1]

\REQUIRE $R$ data-parallel ranks; global batch size $B_k$;
         $L$ layers with parameters $x_{k,l} \in \mathbb{R}^{m_l \times n_l}$,
         $P_l = m_l n_l$;

\FOR{$l = L, L{-}1, \ldots, 1$ \textbf{(reverse layer order, backward pass)}}

    \STATE \textbf{// --- FSDP Backward ---}
    \STATE \texttt{AllGather}$(x_{k,l}^{j}) \;\to\; x_{k,l}$
           \hfill\COMMENT{Restore full parameters on rank $j$}
    \STATE Compute local unsharded gradient:
           $\;g_{k,l}^j \;=\; \dfrac{R}{B_k}
             \!\sum_{\xi \in \mathcal{S}_k^j} \nabla_{x_{k,l}}\,\ell(x_{k};\,\xi)$

    \STATE \textbf{// --- Pre-ReduceScatter Hook: Per-Layer Variance Statistics ---}
    \STATE Compute local element-wise square:
           $\;z_{k,l}^j \;\leftarrow\; (g_{k,l}^j)^2$
           \hfill\COMMENT{$O(P_l)$ element-wise ops; rank-local only}

    \STATE \textbf{// --- FSDP Gradient Sharding ---}
    \STATE $\bar{g}_{k,l} \;\leftarrow\;
           \texttt{ReduceScatter}\!\left(\tfrac{1}{R}\textstyle\sum_{j=1}^R g_{k,l}^j\right)$
           \hfill\COMMENT{Standard FSDP; shard size $P_l/R$}
    \STATE $h_{k,l} \;\leftarrow\;
           \texttt{ReduceScatter}\!\left(\tfrac{1}{R}\textstyle\sum_{j=1}^R z_{k,l}^j\right)$
           \hfill\COMMENT{1 additional RS; shard size $P_l/R$}

    \STATE \textbf{// --- Post-ReduceScatter: Compute Variance ---}
    \STATE Compute per-coordinate variance for shard $j$ (\cref{eq:GNS_estimate_SGD_signSGD}):
           \[
             (\hat{\sigma}_{k,l}^j)^2 \;=\;
             \frac{B_k}{R-1}\!\left(h_{k,l}^j \;-\; (\bar{g}_{k,l}^j)^2\right),
           \]
    \STATE Store $\ell_1$ norm for the variance of layer $l$ shard $j$:  $\|\hat{\sigma}_{k,l}^j\|_1$
\ENDFOR

\STATE \textbf{// --- Global Aggregation ---}
\STATE Accumulate layer contributions:
        $\|\hat{\sigma}_{k}^j\|_1 = \textstyle\sum_{l=1}^L  \|\hat{\sigma}_{k,l}^j\|_1$, $\|\bar{g}_{k}^j\|_1 = \textstyle\sum_{l=1}^L  \|\bar{g}_{k,l}^j\|_1$.
\STATE $\|\hat{\sigma}_{k}\|_1 \;\leftarrow\;
           \texttt{AllReduce}\!\left( \textstyle\sum_{j=1}^R \|\hat{\sigma}_{k}^{j}\|_1\right)$ \hfill\COMMENT{All reduce a scalar across shards}
\STATE $\|g_{k}\|_1 \;\leftarrow\;
           \texttt{AllReduce}\!\left( \textstyle\sum_{j=1}^R \|\bar{g}_{k}^j\|_1\right)$ \hfill\COMMENT{All reduce a scalar across shards}
\STATE Compute $\ell_1$ GNS:
$
\hat{\mathcal B}_{\ell_1}
=
\frac{
\|\hat{\sigma}_{k}\|_1^2
}{
\|g_k\|_1^2
}$.
        
\end{algorithmic}
\end{algorithm}

\paragraph{Notation for Distributed Collectives.}
Throughout the algorithms, following \citet{ahn2025dion},
distributed collectives are written as explicit tensor shape
transformations:
\begin{itemize}
    \item $\texttt{AllGather}$: restores a full tensor from parameter shards.
    \item $\texttt{ReduceScatter}$: reduces across data-parallel ranks
    and returns a sharded tensor.
    \item $\texttt{AllReduce}$: reduces across data-parallel ranks
    and broadcasts the result to all ranks.
\end{itemize}

\begin{algorithm}[tb]
\caption{$\mathcal{S}_1$ GNS Estimation with FSDP at step $k$}
\label{alg:fsdp_s1_var}
\begin{algorithmic}[1]

\REQUIRE $R$ data-parallel ranks; global batch size $B_k$;
         $L$ layers with parameters $X_{k,l} \in \mathbb{R}^{m_l \times n_l}$,
         $m_l \le n_l$, $P_l = m_l n_l$

\FOR{$l = L, L{-}1, \ldots, 1$ \textbf{(reverse layer order, backward pass)}}

    \STATE \textbf{// --- FSDP Backward ---}
    \STATE \texttt{AllGather}$(X_{k,l}^{j}) \;\to\; X_{k,l}$
           \hfill\COMMENT{Restore full parameters on rank $j$}
    \STATE Compute local unsharded gradient matrix:
           $\;G_{k,l}^j \;=\; \dfrac{R}{B_k}
             \!\sum_{\xi \in \mathcal{S}_k^j} \nabla_{X_{k,l}}\,\ell(X_{k};\,\xi)$
    \STATE \textbf{// --- Pre-ReduceScatter Hook: Per-Layer Gram Statistics ---}
    \STATE Compute local Gram matrix:
           $\;Q_{k,l}^j \;\leftarrow\; G_{k,l}^j \bigl(G_{k,l}^j\bigr)^\top
             \;\in \mathbb{R}^{m_l \times m_l}$
           \hfill\COMMENT{$O(m_l^2 n_l)$ matmul; rank-local only}
    \STATE Aggregate Gram matrices across ranks:
           $\;A_{k,l} \;\leftarrow\;
             \texttt{AllReduce}\!\left(\textstyle\sum_{j=1}^R Q_{k,l}^j\right)$
           \hfill\COMMENT{$m_l^2$ elements; $m_l^2 \ll P_l$}

    \STATE \textbf{// --- FSDP Gradient Sharding ---}
    \STATE $\bar{G}_{k,l} \;\leftarrow\;
           \texttt{ReduceScatter}\!\left(\tfrac{1}{R}\textstyle\sum_{j=1}^R G_{k,l}^j\right)$
           \hfill\COMMENT{Standard FSDP; shard size $P_l/R$}

    \STATE \textbf{// --- Post-ReduceScatter: Compute Row-Covariance ---}
    \STATE Reconstruct full mean gradient:
           $\;\bar{G}_{k,l} \;\leftarrow\; \texttt{AllGather}(\bar{G}_{k,l}^{j})$
           \hfill\COMMENT{$P_l$ elements; GNS-specific AllGather}
    \STATE Compute mean Gram matrix:
           $\;B_{k,l} \;\leftarrow\; \bar{G}_{k,l} \bar{G}_{k,l}^\top \;\in \mathbb{R}^{m_l \times m_l}$
    \STATE Compute unbiased row-covariance proxy (\cref{eq:spec_sample_variance_estimate}):
           \[
             C_{\mathrm{row},k,l} \;=\;
             \frac{B_k}{R-1}\!\left(\frac{1}{R}\,A_{k,l} \;-\; B_{k,l}\right)
           \]
    \STATE Compute and store layer contributions:
            $\|C_{\mathrm{row},k, l}^{1/2}\|_{\mathcal{S}_1}$,
           $\;\|\bar{G}_{k, l}\|_{\mathcal{S}_1}$ 
           \hfill
           \COMMENT{$O(m_l^3)$ eigendecompositions; local on each rank}
\ENDFOR
\STATE \textbf{// --- Global Aggregation ---}
\STATE Accumulate layer contributions:
        $\|C_{\mathrm{row},k}^{1/2}\|_{\mathcal{S}_1} = \textstyle\sum_{l=1}^L  \|C_{\mathrm{row},k, l}^{1/2}\|_{\mathcal{S}_1}$
        , $\|G_{k}\|_{\mathcal{S}_1} = \textstyle\sum_{l=1}^L  \|\bar{G}_{k, l}\|_{\mathcal{S}_1}$.
\STATE Compute $\mathcal{S}_1$ GNS:
$\hat{\mathcal B}_{\mathcal{S}_1}
=
\frac{\|C_{\mathrm{row},k}^{1/2}\|_{\mathcal{S}_1}^2}{\|G_{k}\|_{\mathcal{S}_1}^2}$.
\end{algorithmic}
\end{algorithm}

\end{document}